%% file: arxiv_main.tex
\documentclass[onecolumn]{main}

\paperstyle{simple}
\papercolor{blue}

\title{PrefixGuard: Online Failure Warning and Trace-Grounded Diagnosis for LLM Agents}

\author[1]{Xinmiao Huang}
\author[1]{Jinwei Hu}
\author[1]{Qisong He}
\author[1]{Rajarshi Roy}
\author[1]{Changshun Wu}
\author[1,\dagger]{Yi Dong}
\author[1,\dagger]{Xiaowei Huang}

\affiliation[1]{University of Liverpool}
\contribution[\dagger]{Corresponding authors}

\abstract{\input{content/abstract.tex}}

\date{September 25, 2026}
\correspondence{\email{xiaowei.huang@liverpool.ac.uk}, \email{yi.dong@liverpool.ac.uk}}

\usepackage{wrapfig}        
\usepackage{float}          
\usepackage{tikz}
\usetikzlibrary{arrows.meta,positioning,fit,calc,backgrounds}
\usepackage{nicefrac}       
\usepackage{listings}       

\definecolor{promptbg}{RGB}{246,247,250}
\definecolor{promptrule}{RGB}{175,182,200}
\definecolor{prompttitle}{RGB}{55,65,100}
\definecolor{rqboxbg}{RGB}{248,249,252}
\definecolor{rqboxrule}{RGB}{182,190,208}
\definecolor{rqboxtitle}{RGB}{45,55,90}

\lstdefinestyle{promptstyle}{
  basicstyle=\ttfamily\footnotesize,
  breaklines=true,
  breakatwhitespace=false,
  columns=fullflexible,
  keepspaces=true,
  backgroundcolor=\color{promptbg},
  frame=lines,
  framerule=0.5pt,
  rulecolor=\color{promptrule},
  xleftmargin=10pt,
  xrightmargin=6pt,
  aboveskip=6pt,
  belowskip=6pt,
  captionpos=t,
  abovecaptionskip=0pt,
  belowcaptionskip=2pt,
}

\DeclareCaptionFormat{promptcap}{%
  \colorbox{prompttitle}{%
    \parbox{\dimexpr\linewidth-2\fboxsep-20pt\relax}{%
      \color{white}\bfseries\small #1#2#3%
    }%
  }\vspace{-1pt}%
}
\newcommand{\rqfinding}[2]{%
  \par\vspace{0.2em}
  \noindent\begingroup
  \setlength{\fboxsep}{5pt}%
  \setlength{\fboxrule}{0.5pt}%
  \fcolorbox{rqboxrule}{rqboxbg}{%
    \begin{minipage}{\dimexpr\linewidth-2\fboxsep-2\fboxrule\relax}
      \textbf{\color{rqboxtitle}#1.} #2
    \end{minipage}%
  }%
  \endgroup
  \par\vspace{0.2em}
}

\newcommand{\xinmiao}[1]{}
\newcommand{\jinwei}[1]{}
\newcommand{\claude}[1]{}

\begin{document}
\maketitle

\addtocontents{toc}{\protect\setcounter{tocdepth}{-1}}
\input{content/intro.tex}
\input{content/method.tex}
\input{content/experiments.tex}

\input{content/case_study.tex}
\input{content/related_work.tex}
\input{content/limitations_conclusion.tex}
\label{iclr:main_end}
\input{content/statements.tex}

\bibliographystyle{archive/common/minimalnat}
\bibliography{reference}

\clearpage

\input{appendix/main.tex}

\end{document}

%% file: content/abstract.tex
Large language model (LLM) agents execute multi-step tasks, but terminal verdicts arrive too late for intervention. Agent traces mix messages, tool calls, and feedback, making it difficult to learn failure signals from terminal outcomes. Even if a model learns to predict failure from execution prefixes, its predictions alone cannot show which actions, feedback, or unmet conditions need inspection. Existing work studies online failure prediction and diagnosis from completed trajectories, yet few methods connect timely warning to a compact symbolic model of execution behavior. We introduce \textbf{PrefixGuard}, a neuro-symbolic method that uses learned events to connect online failure warning with finite-state execution diagnosis. A gated recurrent unit (GRU) predicts horizon-specific failure risk from these events. A deterministic finite automaton (DFA) organizes their histories into compact risk-labeled paths for replay and rule-monitor composition. Across all 24 benchmark--horizon settings, PrefixGuard exceeds the strongest evaluated baseline in mean test area under the precision--recall curve (AUPRC). Average gains range from 12.0 to 21.9 percentage points. At a nominal 10\% FAR budget, averaged test recall spans 56.1\%--93.0\% across benchmarks. On one frozen $\tau^2$-Bench DFA, model checking five specifications identifies violations in 21 of 265 benchmark-successful runs.

%% file: content/intro.tex
\section{Introduction}
\label{sec:intro}

LLM agents carry out multi-step tasks in software engineering, browser navigation, and tool-mediated workflows~\citep{yang2024sweagent,zhou2024webarena,barres2025tau2bench}. Delayed recognition of problems in these workflows can waste resources and allow errors to propagate through subsequent actions. Like an aircraft's black box, recorded traces and terminal verdicts support posthoc analysis, while a cockpit alarm instead warns when there is still time to act (Figure~\ref{fig:concept-overview}a). We ask whether these records can train an agent's cockpit alarm: a monitor that warns from the execution prefix alone. Effective monitoring should identify failure risk before execution ends, giving an operator time to intervene. After an alert, it should help locate the actions, feedback, and unmet task conditions that warrant inspection.
\input{figures/main/fig_concept_overview.tex}

Meeting these needs poses two challenges. \textbf{(1) Learning prefix risks from terminal outcomes.} LLM-agent traces mix natural-language messages, tool calls, and environment feedback, while supervision often only consists of a final success or failure label. The challenge is how to learn which patterns in these varied execution histories signal failure risk before the run ends. \textbf{(2) Diagnosing execution behavior and risk.} Understanding success and failure requires examining how an agent's behavior advances the task or leaves necessary conditions unresolved. For example, the trajectories in Figure~\ref{fig:concept-overview}b share an initial execution prefix, then diverge and reach different outcomes. The challenge is how to capture execution behavior in a compact symbolic model that supports diagnosis through inspection of the recorded actions, feedback, and task requirements.

Research on process monitoring and formal methods offers complementary tools for prediction and inspection. Neural predictive monitors learn to forecast how incomplete executions will continue and end~\citep{teinemaa2019outcome,tax2017predictive}. Automata extraction can represent recurrent behavior through explicit state transitions~\citep{weiss2018extracting}, and runtime verification can check executions against specified requirements~\citep{leucker2009brief}. Work on LLM agents has begun to study online failure prediction and error localization~\citep{zhang2026agentforesight}, as well as failure diagnosis from completed trajectories~\citep{cemri2025why}. Yet few methods connect online failure prediction with a compact symbolic model for inspecting execution behavior and task requirements. The gap is how to learn both capabilities from agent traces and terminal outcomes while retaining the behavioral distinctions needed for diagnosis.

To address these challenges, we introduce \textbf{PrefixGuard}, a neuro-symbolic method that uses learned events to connect online failure warning with finite-state execution diagnosis (Figure~\ref{fig:method_overview}). Each agent step is converted into a structured action-and-feedback record, then mapped to a small set of event categories. We derive targets for intermediate prefixes from terminal outcomes and jointly train this event mapping with a \emph{gated recurrent unit (GRU)}. During execution, the GRU uses its fixed-size recurrent state to update failure risk within a chosen warning horizon after each step. From the same events and prefix targets, we induce a compact \emph{deterministic finite automaton (DFA)} whose explicit states group histories with compatible risks. The DFA replays an execution as a state path and can be combined with rule monitors to check task requirements. The shared events connect the neural warning to an inspectable symbolic account of execution behavior.

We evaluate warning quality and the DFA's diagnostic capabilities through three questions.
\textbf{RQ1:} Can terminal-outcome supervision support timely failure warnings (Section~\ref{sec:deployment})?
\textbf{RQ2:} Can compact DFAs retain predictive performance and support execution analysis (Section~\ref{sec:finite_state_audit})?
\textbf{RQ3:} Can a learned DFA support task-grounded execution diagnosis through model checking (Section~\ref{sec:case_studies})?

Across all 24 benchmark--horizon settings, PrefixGuard exceeds the strongest baseline in mean test AUPRC; benchmark gains average 12.0--21.9 percentage points over horizons. At a nominal 10\% FAR budget, horizon-averaged test recall spans 56.1\%--93.0\% across benchmarks. The same learned events yield DFAs averaging 5.0--32.3 states that retain 90.9\% of paired GRU test AUPRC on average. PrefixGuard DFA replay locates the failed run's first alert and shows how different actions after a shared prefix lead to different outcomes. We conduct model checking on the learned DFA for five retail tool-use rules and validate flagged violations against the original traces. Across 359 $\tau^2$-Bench retail executions, we identify violations in 21 of 265 benchmark-successful runs. Together, these results show that terminal outcomes can train a monitor to flag risky prefixes, while the same learned events yield a compact DFA for distinguishing divergent behavior and checking task requirements.

Our contributions are summarized as follows:
\begin{itemize}[leftmargin=*,itemsep=1pt,topsep=2pt,parsep=0pt]
\item We formulate online agent monitoring around two linked goals: warning of failure before termination and providing a replayable execution path for diagnosis.
\item We introduce PrefixGuard, a neuro-symbolic method that learns events for recurrent failure prediction and uses those events to induce a compact, risk-labeled DFA.
\item We evaluate PrefixGuard on four agent benchmarks, measuring warning timeliness, false alarms, and how much predictive performance the learned DFAs retain.
\item We show how model checking a learned PrefixGuard DFA can locate agents that violate explicit task requirements, with concrete cases validated against the original traces.
\end{itemize}

%% file: figures/main/fig_concept_overview.tex
\begin{figure}[t]
\centering
\includegraphics[width=.92\linewidth,trim=0 4pt 0 0,clip]{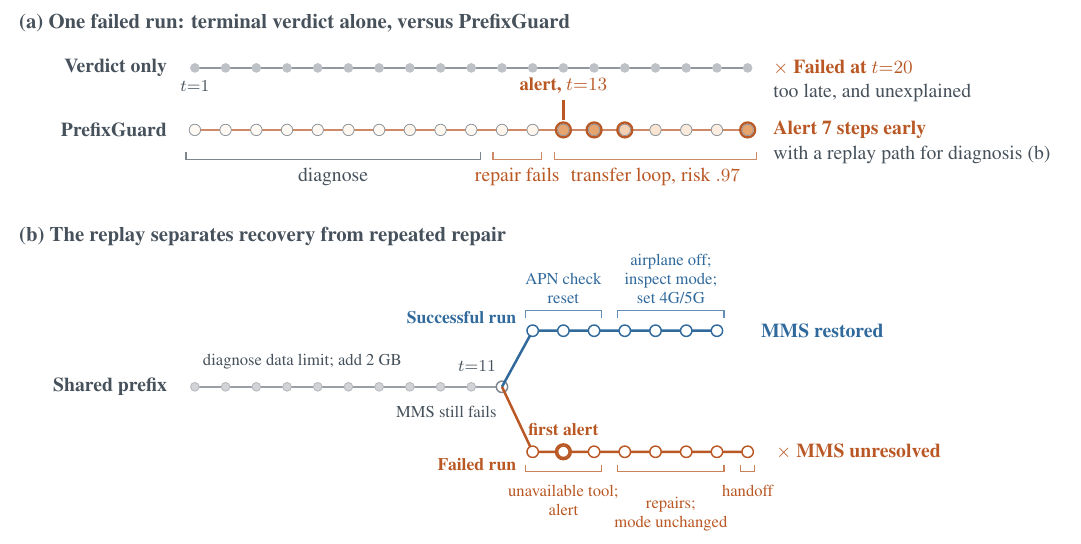}
\caption{\textbf{Warn early, then explain the execution.}
(a) A terminal verdict reveals failure only at $t{=}20$, whereas PrefixGuard first alerts at $t{=}13$ and retains a replayable path.
(b) A replay of two same-task runs: after the shared diagnosis and data refill, the successful run checks APN and network mode before restoring MMS; the failed run calls an unavailable tool, repeats repairs without changing the network mode, and hands off unresolved.
Figure~\ref{fig:execution-divergence} expands this schematic on the learned DFA.}
\label{fig:concept-overview}
\end{figure}

%% file: content/method.tex
\section{PrefixGuard}
\label{sec:method}

\subsection{Problem Formulation}
\label{sec:problem}
\label{sec:prefix_risk_task}

We formulate online agent monitoring as two connected tasks: predicting failure risk from execution prefixes and representing those prefixes as paths in a risk-labeled DFA for diagnosis.

\textbf{Observed execution.}
For task $q$, an agent produces actions $a_t$ and receives observations $o_t$ from the environment.
Each step is $c_t=(a_t,o_t)$, and the observed prefix is $\tau_{:t}=(q,c_1,\ldots,c_t)$.
The trajectory $\tau=\tau_{:T}$ ends with verifier outcome $y\in\{0,1\}$, where $y=1$ denotes success.
During execution, the monitor has access only to this prefix, not future steps or the terminal verdict.

\textbf{Online failure warning.}
Given $\tau_{:t}$ and a chosen \emph{warning horizon} $H$, the problem is to produce a risk score $r_t=f_H(\tau_{:t})\in[0,1]$ estimating the probability that the run ends in failure within $H$ further steps.
For example, $H=3$ asks about failure within three more steps, whereas $H=\infty$ asks about eventual failure regardless of when it occurs.
A warning is issued when $r_t$ first reaches a threshold $\gamma$.
The goal is to warn before failed runs terminate while limiting false alarms on successful runs.
The horizon specifies how far ahead to predict, not a guaranteed warning lead.

\textbf{Finite-state abstraction for diagnosis.}
We seek a learned event alphabet $\Sigma$ shared by the predictor and a compact risk-labeled DFA $M=(Q,\Sigma,\delta,s_0,\rho)$.
Each observed step $c_t$ receives an event symbol $z_t\in\Sigma$; the symbols group steps for failure prediction rather than naming tool operations, so different actions may share one symbol.
Here $Q$ is the state set, $s_0$ its start state, $\delta:Q\times\Sigma\to Q$ the transition function, and $\rho:Q\to[0,1]$ the horizon-specific state risk.
Replaying the events gives $s_t=\delta(s_{t-1},z_t)$ and a state path $s_0,\ldots,s_t$ to compare with recorded actions and feedback.
For example, replaying the learned events shows a shared initial path followed by a failed path that reaches alert state $s_6$ at $t{=}13$ and returns there after further repair attempts, whereas the successful path reaches low-risk $s_8$ (Figure~\ref{fig:execution-divergence}).
Aligning these paths with the original actions and feedback shows that the failed run leaves the network mode at 2G with MMS unresolved, while the successful run selects 4G/5G and restores MMS.

\subsection{Learning a Risk Monitor}
\label{sec:prefixguard}
\label{sec:stepview}
\label{sec:encoder}
\label{sec:abstraction}
\label{sec:monitor}
\label{sec:training}

Figure~\ref{fig:method_overview} summarizes PrefixGuard: a shared event representation connects the recurrent risk predictor to the DFA used for execution inspection.

\input{figures/main/fig_method_overview.tex}

\textbf{Step representation.}
For each execution step, we extract metadata, observation, action, tool, arguments, result, and status into a \textsc{StepView} record.
An LLM proposes the field-extraction parser from training trajectories and we fix it before training the monitor, so no LLM is called at scoring time (Appendix~\ref{app:stepview_induction_repro}).
We encode each record as a TF-IDF~\citep{salton1988term} vector $\mathbf e_t$ of weighted term frequencies, with the vectorizer fitted on training data and kept fixed thereafter.

\textbf{Risk-aligned event learning.}
A projection $g_\phi$ produces unnormalized event scores (logits) $\boldsymbol\ell_t=g_\phi(\mathbf e_t)$ over $K=|\Sigma|$ discrete event labels.
These labels have no predefined tool or action meaning; we interpret them by aligning their occurrences with recorded steps (Section~\ref{sec:finite_state_audit}).
At prediction time, the GRU receives \emph{soft} event probabilities $\mathbf a_t$, while DFA replay uses the most likely, or \emph{hard}, event $z_t$:
\begin{equation}
    \mathbf a_t=\mathrm{softmax}(\boldsymbol\ell_t/T_{\mathrm g}),
    \qquad z_t=\arg\max_k\ell_{t,k}.
    \label{eq:events}
\end{equation}
The temperature $T_{\mathrm g}$ controls how concentrated the event probabilities are; lower values give sharper assignments.
During training, we replace $\mathbf a_t$ with Gumbel-softmax samples $\widetilde{\mathbf a}_t$ at the same temperature, adding noise while retaining differentiability~\citep{jang2017categorical}.
Prediction-loss gradients flow through the GRU into the projection, making the event partition \emph{risk-aligned}.

\textbf{Recurrent prediction.}
The GRU~\citep{cho2014learning} updates $\mathbf h_t=\mathrm{GRU}(P(\mathbf a_t),\mathbf h_{t-1})$ and predicts $r_t=\sigma(\mathbf w^{\top}\mathbf h_t+b)$, with $\mathbf h_0=\mathbf 0$.
Here $\mathbf h_t$ summarizes the observed event history, $P$ projects event probabilities into the GRU input space, $\mathbf w$ and $b$ are the readout weights and bias, and $\sigma$ is the sigmoid function; all model parameters are learned jointly with $g_\phi$.
The recurrent state allows the same event to imply different risks after different histories, at a cost that stays constant in $t$.

\textbf{Training objective.}
From completed training trajectories $\mathcal D_{\mathrm{train}}$ and their terminal outcomes, we derive horizon-specific targets $p_t$ without step-level annotations (Appendix~\ref{app:prefix_labels}).
The event projection and GRU minimize binary cross-entropy against these targets:
\begin{equation}
    \mathcal L_{\mathrm{pred}}=
    \mathbb E_{(\tau,y)\sim\mathcal D_{\mathrm{train}}}
    \left[\frac{1}{T}\sum_{t=1}^{T}\mathcal L_{\mathrm{BCE}}(r_t,p_t)\right].
    \label{eq:objective}
\end{equation}
Averaging within trajectories weights them equally regardless of length.
To encourage confident assignments while keeping multiple event labels in use, we add a balance loss and optimize $\mathcal L=\mathcal L_{\mathrm{pred}}+\lambda\mathcal L_{\mathrm{balance}}$ with $\lambda=0.1$, where
\begin{equation}
    \mathcal L_{\mathrm{balance}}=
    \mathbb E_t[\mathcal H(\widetilde{\mathbf a}_t)]
    -\beta\mathcal H(\mathbb E_t[\widetilde{\mathbf a}_t]).
    \label{eq:balance}
\end{equation}
Entropy $\mathcal H$ is taken over non-padding minibatch steps, with $\beta=1$: the first term penalizes a diffuse assignment at a single step, the second rewards variety across the batch.
Validation AUPRC selects the checkpoint, which is frozen before calibration and finite-state induction.

\subsection{Learning a Finite-State Model for Diagnosis}
\label{sec:dfa}

We build the DFA of Section~\ref{sec:problem} from the frozen monitor's hard events $z_t$ and prefix labels $p_t$, adapting ALERGIA's state merging~\citep{carrasco1994learning} from symbol frequencies to prefix risks: build a prefix tree, merge states whose risks pass a Hoeffding test, and complete missing transitions (Algorithm~\ref{alg:risk_merge}).

\input{content/algorithm.tex}

\textbf{Prefix tree.}
We first organize the hard event sequences $z_t$ into a prefix tree, where sequences with a shared prefix share the same path.
At each node, $n$ counts prefixes ending there (its \emph{support}) and $m$ those with $p_t{=}1$; $\widehat r{=}m/n$ is the fraction that fail within the horizon.
Keeping every prefix would make the DFA large, so we merge prefixes with similar risks.

\textbf{Compatibility.}
Two prefixes should share a state only if they have similar risk and the same next event leads to similar risks from both, because a merged state would otherwise erase differences that matter for diagnosis.
The compatibility test therefore checks nodes recursively under each shared event, comparing empirical risks with $|\widehat r_u-\widehat r_v|\leq\sqrt{\tfrac{1}{2}\log\frac{2}{\eta}}\bigl(n_u^{-1/2}+n_v^{-1/2}\bigr)$, where $n_u,n_v$ are the prefix counts at each node~\citep{hoeffding1963probability,carrasco1994learning}.
Larger $n$ makes the allowed difference smaller, so two prefixes that appear in many runs are merged only when their risks are close (merge tolerance $\eta{=}0.05$; Appendix~\ref{app:risk_merge_details}).
Nodes with no support pass trivially.

\textbf{Merging.}
Starting from the root $s_0$ as the sole DFA state, each round picks the tree node $v$ with the largest support among successors of current DFA states that have not yet been processed.
If no current DFA state is compatible, $v$ becomes a new state; otherwise it is folded into the closest-risk compatible state $u$: the incoming edge is redirected to $u$, counts are summed, and successors of $v$ are merged into or grafted onto $u$.
When the edge starts at $u$ itself, the fold creates a self-loop.
After all nodes are processed, unreachable states are dropped and every missing transition leads to a default state that loops to itself on all events.
Ties throughout are broken by support, then node index.

\textbf{State risks and replay.}
The training-prefix risks $\widehat r$ guide merging only; the final state risks $\rho(s)$ are estimated separately from calibration prefixes, with smoothing as described in Appendix~\ref{app:risk_merge_details}.
Replaying an execution yields a risk-labeled state path.
To diagnose a run, we compare this path with the recorded actions, feedback, and task requirements (Section~\ref{sec:case_studies}).

%% file: figures/main/fig_method_overview.tex
\begin{figure}[!ht]
\centering
\resizebox{\linewidth}{!}{%
\begin{tikzpicture}[
  font=\fontsize{8}{9.5}\selectfont, >={Latex[length=1.5mm]},
  box/.style={draw=black!40, rounded corners=2pt, align=center,
    minimum height=9mm, inner sep=4pt, fill=white, line width=.6pt},
  gru/.style={box, draw=blue!50!black, fill=blue!4},
  dfa/.style={box, draw=orange!70!black, fill=orange!7},
  note/.style={draw=black!25, dashed, rounded corners=2pt, fill=black!2,
    align=center, inner sep=4pt, font=\fontsize{7}{8.5}\selectfont, text=black!65},
  arr/.style={->, line width=.65pt}
]
\node[box, text width=1.8cm] (step) at (0,0)
  {Observed step\\$c_t=(a_t,o_t)$};
\node[box, text width=1.55cm] (sv) at (2.15,0)
  {StepView\\TF-IDF $\mathbf e_t$};
\node[box, text width=1.75cm] (proj) at (4.45,0)
  {Event projection\\$\boldsymbol\ell_t=g_\phi(\mathbf e_t)$};
\node[gru, text width=2.55cm] (gru) at (7.8,.8)
  {GRU update\\$\mathbf h_{t-1}\rightarrow\mathbf h_t$\\Predicted risk $r_t$};
\node[dfa, text width=2.55cm] (dfa) at (7.8,-.8)
  {DFA transition\\$s_t=\delta(s_{t-1},z_t)$\\State risk $\rho(s_t)$};
\node[gru, text width=1.9cm] (warn) at (11.25,.8)
  {Online warning\\$r_t\geq\gamma_{\mathrm G}$};
\node[dfa, text width=1.9cm] (audit) at (11.25,-.8)
  {Risk-labeled\\replay path};
\draw[arr] (step) -- (sv);
\draw[arr] (sv) -- (proj);
\coordinate (fork) at (5.85,0);
\draw[line width=.65pt] (proj.east) -- (fork);
\fill (fork) circle (1pt);
\draw[arr, blue!50!black] (fork) |- (gru.west);
\node[anchor=south east, inner sep=2pt, text=blue!50!black,
  font=\fontsize{7}{8}\selectfont] at (gru.west) {soft $\mathbf a_t$};
\draw[arr, orange!70!black] (fork) |- (dfa.west);
\node[anchor=north east, inner sep=2pt, text=orange!70!black,
  font=\fontsize{7}{8}\selectfont] at (dfa.west) {hard $z_t$};
\draw[arr, blue!50!black] (gru) -- (warn);
\draw[arr, orange!70!black] (dfa) -- (audit);
\node[note, text width=4.15cm] (loss) at (2.1,-2.2)
  {Joint training on $\mathcal D_{\mathrm{train}}$ with prefix labels $p_t$\\
   $\mathcal L=\mathcal L_{\mathrm{pred}}+\lambda\mathcal L_{\mathrm{balance}}$};
\node[note, text width=3.55cm] (induce) at (7.3,-2.2)
  {DFA induction on $\mathcal D_{\mathrm{train}}$\\
   Hard events + prefix labels $p_t$\\Risk-compatible merging};
\node[note, text width=2.4cm] (cal) at (11.25,-2.2)
  {Calibration: fit $\rho$;\\set GRU threshold $\gamma_{\mathrm G}$};
\node[align=center, text=black!60, font=\fontsize{7}{8.5}\selectfont]
  at (2.1,-1.05) {Frozen adapter and vectorizer\\No online LLM};
\end{tikzpicture}}
\caption{\textbf{PrefixGuard architecture.}
A shared projection supplies soft events for GRU warnings and hard events for DFA replay.
DFA paths are inspected alongside recorded actions and feedback.
Bottom annotations summarize training, induction from prefix labels, and calibration.}
\label{fig:method_overview}
\end{figure}

%% file: content/algorithm.tex
\begin{algorithm}[!ht]
\caption{Risk-compatible DFA induction from labeled prefixes}
\label{alg:risk_merge}
\begin{algorithmic}[1]
\Require Ordered labeled prefixes $\mathcal P=[(w_i,p_i)]_{i=1}^N$ over $\Sigma$, with $p_i\in\{0,1\}$; $\eta\in(0,1)$
\State Build prefix tree in input order; assign node indices on creation and count $(m_s,n_s)$ at endpoints
\State Let $\widehat r_s=m_s/n_s$ if $n_s>0$, else $0$
\State $R\gets\{s_0\}$; $B\gets\Call{Candidates}{R}$ \Comment{successors of $R$ outside $R$}
\While{$B\neq\emptyset$}
    \State $v\gets\arg\min_{v\in B}(-n_v,v)$ \Comment{largest support, then smallest index}
    \State $C\gets\{u\in R:\Call{Compatible}{u,v,\eta}\}$
    \If{$C=\emptyset$}
        \State $R\gets R\cup\{v\}$
    \Else
        \State $u\gets\arg\min_{u\in C}(|\widehat r_u-\widehat r_v|,-n_u,u)$
        \State Redirect the edge from $R$ into $v$ to $u$; \Call{Fold}{$u,v$}
    \EndIf
    \State $B\gets\Call{Candidates}{R}$
\EndWhile
\State Compact states reachable from $s_0$; add an absorbing sink for all missing transitions
\State \Return $(Q,\Sigma,\delta,s_0)$
\end{algorithmic}
\end{algorithm}

%% file: content/experiments.tex
\section{Experiments}
\label{sec:experiments}

\subsection{Experimental Setup}
\label{sec:setup}

\textbf{Datasets.}
We evaluate PrefixGuard on traces from four benchmarks: WebArena~\citep{zhou2024webarena}, $\tau^2$-Bench~\citep{barres2025tau2bench}, SkillsBench~\citep{li2026skillsbench}, and TerminalBench~\citep{merrill2026terminalbench}, covering 17 distinct agent implementations and 39 backbone models in total.
Each benchmark's trajectories are split 70/10/10/10 into train, calibration, validation, and test: train fits the encoder and monitor, calibration sets alert thresholds and DFA state risks, validation selects checkpoints, and test is used only for evaluation (Appendix~\ref{app:datasets}).

\textbf{Baselines.}
We compare against a predictive-process-monitoring activity LSTM (PPM-LSTM)~\citep{teinemaa2019outcome,tax2017predictive} and three log-anomaly detectors: PCA on activity counts with squared prediction error (PCA (SPE))~\citep{xu2009detecting}, DeepLog~\citep{du2017deeplog}, and LogBERT~\citep{guo2021logbert}; figures label PrefixGuard's GRU monitor PG-GRU.
We also include a zero-shot DeepSeek-V4-Pro judge scoring a sample of 200 $H{=}3$ test prefixes as a reference point (Appendix~\ref{app:llm_judge}).

The DFA is compared with its paired GRU on the same test prefixes (Section~\ref{sec:finite_state_audit}), with two induction methods on the same frozen events, an unmerged prefix tree and RPNI~\citep{oncina1992rpni} (Appendix~\ref{sec:induction_choice}), and with a differentiable DFA learned end to end under state budgets $Q\in\{8,16,32,64\}$ (Appendix~\ref{app:direct_dfa}).

\textbf{Metrics.}
We measure the online monitor four ways on every benchmark and horizon: AUPRC, how well its risk scores rank failing prefixes above safe ones; recall, the share of failed runs it alerts on; lead, how early the first alert arrives; and false-alarm rate (FAR), the share of successful runs it alerts on.
An alert fires when the monitor's score first exceeds a threshold $\gamma_{\mathrm G}$.
We set $\gamma_{\mathrm G}$ by split-conformal calibration~\citep{angelopoulos2021gentle} on successful calibration runs at a nominal FAR budget of 10\% (Appendix~\ref{app:eval_protocol}).

We report each DFA's AUPRC alongside its paired GRU, and the number of states.
For alert fidelity, we select DFA alert states on calibration runs to match GRU alerts, then report $\Pr(a_D{=}1\mid a_G{=}1)$ and $\Pr(a_G{=}1\mid a_D{=}1)$ and their harmonic mean as fidelity F1 (Appendix~\ref{app:matched_alert_fidelity}).
Unless noted otherwise, results report means across 3 seeds on test splits, with one population standard deviation.

\subsection{Online Failure Warning}
\label{sec:deployment}
\label{sec:rq1}
\input{figures/main/fig_horizon_within_task.tex}

\textbf{RQ1: Can terminal-outcome supervision support timely failure warnings?}
\rqfinding{RQ1}{Terminal-outcome supervision suffices for timely warnings: PrefixGuard exceeds every baseline across all 24 settings, and $H$ trades lead against recall.}

\textbf{PrefixGuard outranks all baselines.}
PrefixGuard achieved higher mean AUPRC than PPM-LSTM and the log-anomaly baselines in all 24 benchmark--horizon settings (Figure~\ref{fig:horizon-tradeoff}, top).
Averaged over horizons, the gains over the strongest baseline, PPM-LSTM in every benchmark, range from 12.0 percentage points on TerminalBench to 21.9 on WebArena.

\textbf{The warning horizon sets a lead--coverage trade-off.}
Larger $H$ consistently increased lead but did not uniformly raise recall (Figure~\ref{fig:horizon-tradeoff}, middle and bottom).
A wider horizon labels more prefixes as at-risk, including ones still far from the eventual failure. These noisier targets make the threshold harder to calibrate, so recall can drop.
This happened on SkillsBench across all finite horizons and on $\tau^2$-Bench at the longest ones, whereas WebArena recall rose gradually.
The $H=\infty$ models offered the greatest lead but predicted eventual rather than imminent failure, making $H$ an operating choice that trades lead for coverage.

\textbf{Baselines provide little usable advance warning.}
Under the same 10\% calibration FAR budget, PPM-LSTM and PCA catch most failed runs on $\tau^2$-Bench, SkillsBench, and TerminalBench (horizon-averaged recall .81--1.00), while DeepLog and LogBERT catch .08--.56 (Figure~\ref{fig:horizon-tradeoff}, middle).
Their alerts, however, arrive just before termination: every baseline's horizon-averaged lead is .03--.12, against .15--.39 for PrefixGuard on all four benchmarks (Figure~\ref{fig:horizon-tradeoff}, bottom).
Realized test FAR is tabulated in Appendices~\ref{app:horizon} and~\ref{app:baseline_deployment}; the split-conformal budget bounds it only in expectation.

\input{tables/section51_component_gains.tex}

\textbf{Zero-shot LLM judging falls well short.}
DeepSeek-V4-Pro reaches $H{=}3$ AUPRC of .450, .396, .080, and .107 on the four benchmarks (Figure~\ref{fig:horizon-tradeoff}, dashed), no better than PCA and far below PrefixGuard, while costing one LLM call per step.

\textbf{Where do the gains come from?}
Table~\ref{tab:analysis_controls} separates two ingredients (Appendix~\ref{app:event_interface_controls}).
Sequential state helps at short horizons: against a pooled StepView MLP without it, the recurrent monitor gains 2.44 and 0.51 points on WebArena and SkillsBench but trails by 2.37 and 1.78 on $\tau^2$-Bench and TerminalBench.
The event projection is not the source of the margin: a GRU fed the same TF-IDF features without it scores 0.4--2.4 points higher still (Figure~\ref{fig:horizon-tradeoff}, hollow squares), 13.3--22.3 above PPM-LSTM.
The gain over baselines thus comes from outcome-supervised recurrence over full step records rather than activity codes, and the projection spends a small part of it on the discrete histories that the DFA tested next requires.

\input{figures/main/fig_dfa_audit_main.tex}

\subsection{Finite-State Fidelity and Execution Diagnosis}
\label{sec:finite_state_audit}
\input{figures/main/fig_execution_divergence.tex}
\textbf{RQ2: Can compact DFAs retain predictive performance and support execution analysis?}
\rqfinding{RQ2}{Compact DFAs retain over 90\% of GRU AUPRC while providing interpretable state paths for execution diagnosis.}

\textbf{The DFA preserves most of the GRU's risk ranking.}
Across benchmarks and horizons, DFAs averaged 5.0--32.3 states and retained 90.9\% of their paired GRU's AUPRC on average (Figure~\ref{fig:dfa-audit}, top two rows).
Alert fidelity (bottom row) is high: with GRU alerts as reference (orange), DFA recall is 71--98\%; with DFA alerts as reference (green), precision is 64--88\% on $\tau^2$-Bench, SkillsBench, and TerminalBench.
On WebArena precision drops to 37--76\%: its few test prefixes leave each DFA state too coarse for the GRU's continuous threshold (Appendix~\ref{app:dfa_own_far10}).

\textbf{The DFA provides interpretability.}
We replay agent traces through the event projection and DFA to obtain a state path, then align the path with the recorded actions and feedback.
Figure~\ref{fig:execution-divergence} illustrates an example of a DFA graph on two $\tau^2$-Bench runs ($H{=}3$, seed 13): the failed run reaches alert state $s_6$ at $t{=}13$ (risk .965), passes through further repair states, and returns to $s_6$ at $t{=}19$, while the successful run stays in low-risk $s_8$ as it restores MMS.
The state path locates where the runs diverge, and the aligned actions and feedback reveal why---the failed run leaves the network mode unchanged.

%% file: figures/main/fig_horizon_within_task.tex
\begin{figure}[!ht]
  \centering
  \includegraphics[width=\linewidth]{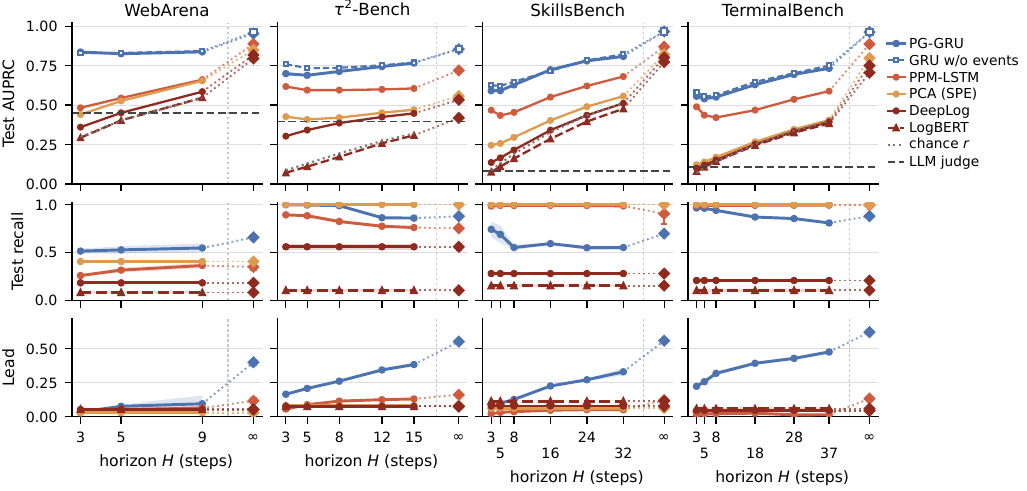}
  \caption{\textbf{Warning across horizons.} Test AUPRC (top), recall (middle), and lead (bottom) for PrefixGuard and four baselines; hollow squares: a GRU on the same StepView features without the event projection (Table~\ref{tab:analysis_controls}). PPM-LSTM uses \texttt{tool|status} activities; DeepLog and LogBERT are scored by mean NLL. Each uses a 10\% calibration FAR threshold; realized test FAR is in Appendix~\ref{app:horizon}. Gray dotted: test prevalence; gray dashed: $H{=}3$ LLM-judge reference.}
  \label{fig:horizon-tradeoff}
  \label{fig:baseline-horizon}
\end{figure}

%% file: tables/section51_component_gains.tex
\begin{wraptable}{R}{0.42\textwidth}
\centering\footnotesize\setlength{\tabcolsep}{2.4pt}
\caption{\textbf{Event ablations.} Mean test $\Delta$AUPRC (pp) over horizons and three seeds: event GRU minus pooled StepView MLP without sequential state; event GRU minus direct-input GRU.}
\label{tab:analysis_controls}
\begin{tabular}{@{}lrr@{}}
\toprule
Benchmark & History & \shortstack{Event\\vs. cont.} \\
\midrule
WebArena & +2.44 & -0.36 \\
$\tau^2$-Bench & -2.37 & -2.38 \\
SkillsBench & +0.51 & -1.37 \\
TerminalBench & -1.78 & -1.26 \\
\bottomrule
\end{tabular}
\end{wraptable}

%% file: figures/main/fig_dfa_audit_main.tex
\begin{figure}[!h]
  \centering
  \includegraphics[width=\linewidth]{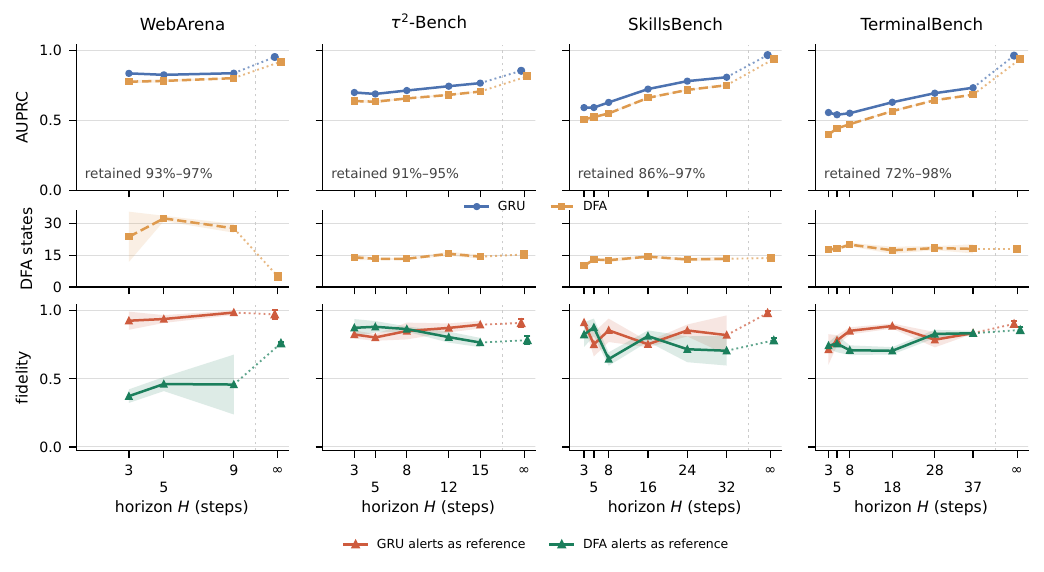}
  \caption{\textbf{Finite-state audit across horizons.}
    \textbf{Rows:} test AUPRC of the GRU and its paired DFA (insets: DFA/GRU retention); DFA state count including the sink; prefix-level alert recall (orange) and precision (green) of the DFA relative to the GRU.
    Alert states are matched on calibration data without a FAR constraint.}
  \label{fig:dfa-audit}
\end{figure}

%% file: figures/main/fig_execution_divergence.tex
\begin{figure}[!ht]
  \centering
  \includegraphics[width=\linewidth]{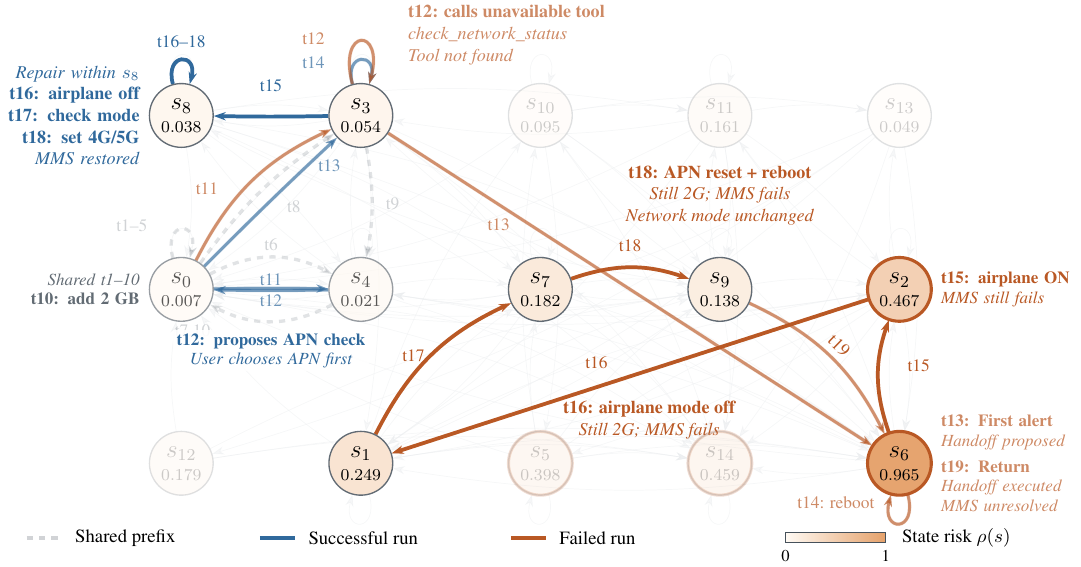}
  \caption{\textbf{Risk-labeled DFA paths make warnings inspectable.}
  In the failed run, $s_6$ marks the first alert and a later return; the successful run stays in low-risk $s_8$ during recovery.
  Node fill shows calibrated $H{=}3$ risk; orange rings mark alert states, and unused transitions are faded.}
  \label{fig:execution-divergence}
\end{figure}

%% file: content/case_study.tex
\section{Case Study: Task-Grounded Diagnosis with Model Checking}
\label{sec:case_studies}

\textbf{RQ3: Can a learned DFA support task-grounded execution diagnosis through model checking?}
\input{tables/model_checking_histories.tex}

\textbf{Motivation.}
Section~\ref{sec:finite_state_audit} showed that the DFA's coarse states can group prefixes with different risks.
We argue that domain knowledge can further distinguish prefixes within the same coarse state---for example, whether the agent verified a change before making another one.
To test this, we write a safety specification over abstract event classes, compile it into a monitor automaton, and model-check it against the $\tau^2$-Bench DFA from Section~\ref{sec:finite_state_audit} via product composition.
The specification requires a read-class event between every two write-class events.
Table~\ref{tab:spec_histories} shows the result: three prefixes reaching the same DFA state $s_9$ are separated by the monitor into ready, awaiting a read, and violated, with calibration risks of 0.098, 0.158, and 0.308.

Adding this specification raised DFA AUPRC from .457 to .510 (Appendix~\ref{app:abstract_event_specs}), confirming that product composition can refine the DFA's coarse risk estimates.
However, abstract event classes cannot explain \emph{why} a run fails.
Diagnosing specific failures requires task-grounded specifications over concrete tool calls.
Below we apply the same composition mechanism with such specifications to audit execution outcomes.

\textbf{Setup.}
We reuse the same $\tau^2$-Bench DFA, keeping its events and transitions fixed.
The study covers 359 retail executions, including 265 successes and 94 failures across 114 tasks and four agents.

\textbf{Requirements from tool policies.}
Inspection of the failures yielded five families covering order status ($\varphi_1$), replacement pairs ($\varphi_2$), item quantities ($\varphi_3$), available same-product variants ($\varphi_4$), and refund destinations ($\varphi_5$). All definitions were fixed before whole-population counting.

\input{tables/process_conformance_audit.tex}
\textbf{Formal specifications.}
We encode each requirement as an $\mathrm{LTL}_f$ safety property~\citep{degiacomo2013ltlf} (full notation in Appendix~\ref{app:failure-derived-audit}):
$\varphi_1$~checks that the order status permits the operation (e.g.\ only cancel a pending order);
$\varphi_2$~checks that replacement lists match in length and that modifications change the item;
$\varphi_3$~checks that the requested quantity does not exceed stock;
$\varphi_4$~checks that the replacement item is an available same-product variant;
$\varphi_5$~checks that the refund goes to the original payment method or the owner's gift card.
All properties except $\varphi_2$ require retrieving order history before the action; runs that never trigger a property satisfy it vacuously.

\textbf{Model checking on PrefixGuard DFA.}
We annotate DFA edges with tool-call facts and compose the DFA with each property's monitor automaton via product composition~\citep{baier2008principles}.
Each product state $(s,u)$ pairs the learned risk state $s$ with the monitor state $u$ tracking whether the specification has been satisfied or violated.
All five specifications reach violating states on the learned DFA without retraining it.
Since the DFA can recombine edges into paths not observed in training, we verify each flagged violation against the original trajectory before counting it.

\textbf{Results.}
After verification, the five specifications flag 21 of 265 successes (7.92\%) and 19 of 94 failures (Table~\ref{tab:batch-process-audit}).
Five flagged successes were never rejected by the runtime yet violate a specification, revealing gaps between the tool's enforcement and the stated policy (Appendix~\ref{app:failure-derived-audit}).

\input{figures/main/fig_retail_temporal_case.tex}

\textbf{Case: diagnosing a failed retail task.}
Now we give an example of how the product graph pinpoints a failure cause.
In retail task 27, the customer prefers exchanging boots over returning them if both options cannot be fulfilled.
Figure~\ref{fig:retail-temporal-case} compares a failed and a successful run: the failed agent attempts a return first, which changes the order status and blocks the subsequent exchange; the successful agent completes the exchange first and then handles the return.
Both runs violate $\varphi_1$ (operating on an order whose status no longer permits the action), but only the successful run passes the database check.
The two violations share the same DFA state $s_2$; the monitor state distinguishes them by tracking the status change that the risk state alone cannot capture.
The product graph identifies the blocked precondition, and the task context explains why it matters (Appendix~\ref{app:temporal-requirements}).

\rqfinding{RQ3}{The learned DFA supports task-grounded diagnosis through model checking: product composition with LTL$_f$ specifications flags process violations, refines risk beyond DFA states, and pinpoints failure causes without retraining the DFA.}

%% file: tables/model_checking_histories.tex
\begin{wraptable}{R}{0.5\textwidth}
\vspace{-0.5em}
\centering\small
\caption{\textbf{Specification-based refinement at DFA state $s_9$.}
$\tau^2$-Bench calibration ($H=3$, seed 13); R/W/O = read/write/other.
Risk: fitted 3-step failure probability; DFA-only risk: .138.}
\label{tab:spec_histories}
\setlength{\tabcolsep}{2pt}
\begin{tabular*}{\linewidth}{@{\extracolsep{\fill}}llrr@{}}
\toprule
Monitor state & Event history & Visits & Risk \\
\midrule
Ready & R, O, O, W, R & 283 & .098 \\
Awaiting a read & O, R, O, W, O & 53 & .158 \\
Violated & O, W, O, O, W, O & 69 & .308 \\
\bottomrule
\end{tabular*}
\vspace{-0.5em}
\end{wraptable}

%% file: tables/process_conformance_audit.tex
\begin{wraptable}{R}{0.42\textwidth}
\centering\small
\caption{\textbf{Process verdicts by outcome.} Violated/unknown runs; \emph{Any} counts runs with at least one violation.}
\label{tab:batch-process-audit}
\setlength{\tabcolsep}{2pt}
\begin{tabular*}{\linewidth}{@{\extracolsep{\fill}}lcc@{}}
\toprule
Property & \shortstack{Success\\($n=265$)} & \shortstack{Failure\\($n=94$)} \\
\midrule
$\varphi_1$ & 12/0 & 11/0 \\
$\varphi_2$ & 3/0 & 4/0 \\
$\varphi_3$ & 0/0 & 1/0 \\
$\varphi_4$ & 1/8 & 1/4 \\
$\varphi_5$ & 6/0 & 3/0 \\
\midrule
Any & 21/6 & 19/0 \\
\bottomrule
\end{tabular*}
\vspace{-1em}
\end{wraptable}

%% file: figures/main/fig_retail_temporal_case.tex
\begin{figure}[!ht]
\centering
\resizebox{\linewidth}{!}{\input{figures/main/retail_temporal_case_tikz.tex}}
\caption{\textbf{Task-27 diagnosis on PrefixGuard's DFA.}
Both runs violate the order-status requirement $\varphi_1$, but only the failed run leaves the preferred exchange unmet.
Dashed arrows omit intermediate transitions; task outcomes are annotations, not DFA states.}
\label{fig:retail-temporal-case}
\end{figure}

%% file: figures/main/retail_temporal_case_tikz.tex
\begin{tikzpicture}[
  x=1cm,y=1cm,font=\sffamily\fontsize{8.5}{10}\selectfont,
  >={Stealth[length=1.8mm,width=1.2mm]},
  state/.style={circle,draw=retailink,line width=.65pt,fill=white,
    minimum size=6.5mm,inner sep=0pt,font=\fontsize{10}{11}\selectfont},
  note/.style={align=center,inner sep=1pt},
  flow/.style={->,line width=.8pt},
  result/.style={rounded corners=2pt,draw,line width=.7pt,
    minimum width=2.25cm,minimum height=.82cm,align=center,inner sep=4pt}
]
\definecolor{retailink}{HTML}{46515B}
\definecolor{retailfail}{HTML}{B95824}
\definecolor{retailsuccess}{HTML}{30699C}
\node[note,text=retailink,font=\sffamily\bfseries\fontsize{8.5}{10}\selectfont]
  at (5.50,1.04) {Precondition changes};
\node[note,text=retailink,font=\sffamily\bfseries\fontsize{8.5}{10}\selectfont]
  at (10.00,1.04) {First rule violation};
\node[note,text=retailink,font=\sffamily\bfseries\fontsize{8.5}{10}\selectfont]
  at (12.90,1.04) {Task outcome};

\node[note,text=retailfail,font=\sffamily\bfseries\fontsize{9}{10}\selectfont]
  at (0,0) {Failed\\run};
\node[state] (f7) at (1.50,0) {$s_7$};
\node[state] (f1) at (5.50,0) {$s_1$};
\node[state] (f2) at (10.00,0) {$s_2$};
\node[note,text=retailink] at (1.50,-.60) {$t=10$\\Delivered};
\node[note,text=retailfail] at (5.50,-.60) {Return requested\\Exchange unavailable};
\node[note,text=retailfail] at (10.00,-.60) {$\varphi_1$ violated\\Exchange rejected};
\draw[flow,retailfail] (f7) -- node[note,above=4pt,text=retailink]
  {Return accepted\\$t=11,\ e_5$} (f1);
\draw[flow,retailfail] (f1) -- node[note,above=4pt,text=retailink]
  {Boots exchange attempted\\$t=12,\ e_5$} (f2);
\node[result,draw=retailfail,fill=retailfail!6,text=retailfail] at (12.90,0)
  {Required exchange unmet\\Database check fails};

\node[note,text=retailsuccess,font=\sffamily\bfseries\fontsize{9}{10}\selectfont]
  at (0,-2.0) {Successful\\run};
\node[state] (s0) at (1.50,-2.0) {$s_2$};
\node[state] (s1) at (5.50,-2.0) {$s_2$};
\node[state] (s2) at (10.00,-2.0) {$s_2$};
\node[note,text=retailink] at (1.50,-2.60) {$t=12$\\Delivered};
\node[note,text=retailsuccess] at (5.50,-2.60) {Exchange requested\\Return unavailable};
\node[note,text=retailfail] at (10.00,-2.60) {$\varphi_1$ violated\\Return rejected};
\draw[flow,retailsuccess] (s0) -- node[note,above=4pt,text=retailink]
  {Boots exchange accepted\\$t=13,\ e_5$} (s1);
\draw[flow,retailfail,dashed] (s1) -- node[note,above=4pt,text=retailink]
  {Return attempted\\$t=16,\ e_5$} (s2);
\node[result,draw=retailsuccess,fill=retailsuccess!6,text=retailsuccess] at (12.90,-2.0)
  {Required exchange accepted\\Database check passes};
\end{tikzpicture}

%% file: content/related_work.tex
\section{Related Work}
\label{sec:related}

\textbf{Predictive monitoring and alarms.}
Predictive process monitoring forecasts outcomes and continuations of incomplete executions from event logs, including with LSTM sequence models~\citep{teinemaa2019outcome,tax2017predictive}.
Log-anomaly detection instead identifies unusual event sequences using methods such as PCA, next-event prediction, and masked-event modeling~\citep{xu2009detecting,du2017deeplog,guo2021logbert}.
These two lines motivate our matched PPM-LSTM, PCA, DeepLog, and LogBERT baselines.
Alarm-based prescriptive monitoring additionally selects when to intervene using explicit costs~\citep{fahrenkrog2022fire}; PrefixGuard evaluates warnings but does not optimize interventions.
LLM-as-judge and digital-agent evaluators score observed behavior or task success~\citep{zheng2023judging,pan2024autonomous}.
PrefixGuard instead scores every step online without deployment-time LLM calls and evaluates warnings by lead, recall, and realized false alarms under frozen thresholds.
Concurrently, AgentForesight~\citep{zhang2026agentforesight} trains a 7B auditor with reinforcement learning to localize the decisive error in multi-agent trajectories, targeting step-level attribution rather than failure risk.

\textbf{Runtime verification, specification mining, and trace abstraction.}
Runtime verification checks executions against supplied formal properties~\citep{leucker2009brief,bauer2011runtime}, whereas specification mining infers likely interaction protocols from observed traces~\citep{ammons2002mining}.
Dialog2Flow clusters utterance embeddings into discrete workflow sequences~\citep{burdisso2024dialog2flow}.
PrefixGuard separates these roles: outcome supervision learns a risk-aligned event alphabet and audit graph, while composition with explicit rule monitors checks task requirements.

\textbf{Finite-state learning and neural-model audit.}
Weiss et al.~\citep{weiss2018extracting} extract automata from recurrent networks using the network as an oracle, while Shvo et al.~\citep{shvo2021interpretable} learn finite-state sequence classifiers through discrete optimization.
RPNI merges prefix-tree states consistent with labeled examples~\citep{oncina1992rpni}, ALERGIA those with Hoeffding-compatible frequencies~\citep{carrasco1994learning}; DeepDFA learns a DFA end to end through relaxed transitions~\citep{umili2024deepdfa}.
We compare risk-compatible merging with a prefix tree, RPNI, and a differentiable DFA; the paired GRU measures neural--symbolic fidelity.
PrefixGuard's graph induction uses failure-labeled prefixes rather than GRU outputs as an oracle.

%% file: content/limitations_conclusion.tex
\section{Conclusion and Future Work}
\label{sec:conclusion}
\label{sec:limitations}

PrefixGuard learns a shared event representation from terminal outcomes and uses it for online failure warning and finite-state execution diagnosis.
Across 24 benchmark--horizon settings, the event-based GRU outranks all baselines and a zero-shot LLM judge, and the warning horizon trades lead against recall.
Compact DFAs induced from the same events retain about 91\% of the GRU's AUPRC and expose state paths that replay divergence and support model checking of task-grounded requirements.
Performance on unseen tasks and agents remains unmeasured.
Controlled intervention studies should test whether correcting identified requirement violations improves task outcomes.

%% file: content/statements.tex
\section*{AI Use Statement}
LLMs were used to propose the offline StepView adapters described in Appendix~\ref{app:stepview_induction_repro}. AI assistance was also used to reorganize and edit the manuscript, and format result tables. AI agents also assisted the descriptive replay annotations.

\section*{Reproducibility Statement}
Appendix~\ref{app:implementation} specifies the StepView adapters and their frozen induction protocol (Appendices~\ref{app:stepview_protocol} and~\ref{app:stepview_induction_repro}), the encoder, event projection, and GRU settings with the training and checkpoint-selection recipe (Appendix~\ref{app:event_training_details}, Tables~\ref{tab:app_frontends} and~\ref{tab:app_training_settings}), the risk-compatible induction and audit-readout procedure (Appendix~\ref{app:risk_merge_details}, Algorithm~\ref{alg:app_pipeline}), and recorded compute (Appendix~\ref{app:compute_resources}).
Appendix~\ref{app:datasets} documents the trace sources and agent composition, the data splits, the prefix-label construction for every horizon, and the split-conformal calibration and metric definitions.
Appendices~\ref{app:warning_results}--\ref{app:finite_state_audit} report per-horizon results with seed variation for PrefixGuard, every baseline, the ablations, and the DFA audits, and Appendix~\ref{app:llm_judge} reproduces the LLM-judge prompts.
Appendix~\ref{app:spec_checking} gives the formal specifications, monitor constructions, and independent \textsc{flloat} checks behind the model-checking results.
Source code, the frozen data splits, and the evaluator exports that generate every table and figure will be released with the paper.

%% file: appendix/main.tex
\appendix
\addtocontents{toc}{\protect\setcounter{tocdepth}{2}}
\begingroup
\small
\setlength{\parskip}{0pt}
\renewcommand{\contentsname}{Appendix Contents}
\tableofcontents
\endgroup
\clearpage
\input{appendix/implementation.tex}
\input{appendix/data_protocol.tex}
\input{appendix/warning_results.tex}
\input{appendix/design_analysis.tex}
\input{appendix/audit_replays.tex}
\input{appendix/spec_checking.tex}

%% file: appendix/implementation.tex
\section{Model and Training Details}
\label{app:implementation}

\subsection{StepView: LLM-Induced Adapter, Frozen Before Training}
\label{app:stepview_protocol}

StepView records use benchmark-specific field mappings proposed once by an offline LLM and fixed before monitor training.
Each record contains metadata, a reduced observation, action text, tool name, serialized arguments, result text, and status.
The serializer collapses whitespace in text and encodes arguments as JSON with sorted keys and ASCII escaping.
Missing tool names or statuses become \texttt{NONE}; missing metadata becomes \texttt{NO\_METADATA}; an empty reduced observation renders as \texttt{NONE}.
Empty action and result strings remain empty.
Unrecognized tool names remain text features rather than triggering a separate full-step fallback.

The renderer retains at most eight observation units: dialogue turns for $\tau^2$-Bench, log blocks for TerminalBench, trace-family-specific units for SkillsBench, and lexical lines for WebArena.
The dataset identity determines the serializer.
For WebArena, the encoded text contains metadata, action, tool, arguments, result, and status.
For the other three benchmarks, it also includes the reduced observation (Table~\ref{tab:app_frontends}).
Metadata appears within the serialized text and is not supplied through a separate feature channel.

\begin{table}[!ht]
\centering\small
\caption{Input and model settings. $K$ is the event-alphabet size and $d_h$ is the GRU hidden dimension. The event-projection hidden layer has width 128 in every corpus.}
\label{tab:app_frontends}
\begin{tabular}{llrr}
\toprule
Corpus  & $K$ & $d_h$ \\
\midrule
WebArena  & 16 & 16 \\
$\tau^2$-Bench & 16 & 16 \\
SkillsBench & 16 & 16 \\
TerminalBench & 32 & 32 \\
\bottomrule
\end{tabular}
\end{table}

\subsection{Offline Adapter Induction}
\label{app:stepview_induction_repro}
\label{app:llm_api_checkpoints}

For each benchmark, an LLM proposes how to map the training-trace fields into StepView; the accepted mapping is fixed before model training and applied to all splits.
The LLM runs only offline: GRU scoring and DFA replay use the fixed mapping, which is chosen independently of outcome labels.

For each benchmark, twelve raw steps are selected from the first 64 training trajectories: four initial, four middle, two tool/action, and two anomalous steps.
Within each group, examples are ordered by \texttt{sha256(trajectory\_id:step\_index)}, making the selection independent of labels and model scores.
The proposal uses \texttt{gpt-5.4-nano} at temperature zero with a constrained JSON schema.
The prompt names the seven StepView fields, asks for direct extraction rather than semantic rewriting, requires unavailable fields to remain unknown, and restricts selections to supported source fields and observation-reduction rules.
The proposed mapping is accepted only after schema and field-validity checks, then fixed before training.

\subsection{Encoding, Training, and Checkpoint Selection}
\label{app:event_training_details}

\textbf{Encoding.}
Each serialized step is transformed by a unigram/bigram TF-IDF vectorizer fitted on training text and held fixed afterward.
All runs use the cuML implementation with 4,096 input dimensions.
Step text is encoded without character truncation.

For cuML vocabulary fitting, the implementation guards against 32-bit string-offset overflow.
If $3.5$ times the total input character count exceeds $1.5\times10^9$, it uniformly subsamples training texts with a separate NumPy seed of 42 to meet that estimated budget.
This affects vocabulary fitting only; all steps are transformed afterward.
The fitted vocabulary and TF-IDF weights are held fixed during evaluation.

\textbf{Model.}
For a step vector $\mathbf e_t\in\mathbb R^{4096}$, the event projection is
\[
 \boldsymbol\ell_t= W_2\operatorname{GELU}(W_1\mathbf e_t+b_1)+b_2,
 \qquad W_1\in\mathbb R^{128\times4096},\quad W_2\in\mathbb R^{K\times128}.
\]
Training uses soft Gumbel--Softmax assignments without a straight-through estimator.
The GRU receives the event probabilities through a learned linear projection and GELU, updates a single recurrent layer of width $d_h$, and applies a linear sigmoid risk head.
Table~\ref{tab:app_frontends} gives $K$ and $d_h$.
At evaluation, Gumbel noise is removed:
\[
 \mathbf a_t=\operatorname{softmax}(\boldsymbol\ell_t/T_{\mathrm g}),
 \qquad z_t=\arg\max_k \ell_{t,k}.
\]
The GRU consumes $\mathbf a_t$, whereas DFA induction and replay consume $z_t$.
The evaluation temperature $T_{\mathrm g}$ is restored from the selected checkpoint.

\begin{table}[!ht]
\centering\small
\caption{Shared training settings. The prediction loss uses no class weights.}
\label{tab:app_training_settings}
\begin{tabular}{ll}
\toprule
Setting & Value \\
\midrule
Optimizer & AdamW \\
Learning rate / weight decay & $10^{-3}$ / $10^{-4}$ \\
Epochs / batch size & 24 / 64 trajectories \\
Prediction / balance weights & 1 / 0.1 \\
Marginal-entropy weight & 1 \\
Gumbel temperature & Linear, 1.0 at epoch 1 to 0.25 at epoch 24 \\
Sequence handling & Full trajectories; padded batches and label masks \\
Batch order & Input order; length bucketing disabled \\
Gradient clipping & None in the evaluated training path \\
Checkpoint selection & Validation AUPRC, with tie rules below \\
\bottomrule
\end{tabular}
\end{table}

\textbf{Loss and batching.}
For batch $\mathcal B$, let $V_i$ be the valid labeled positions of trajectory $i$.
The loss is
\begin{align}
 \mathcal L_{\mathrm{pred}}
 &=\frac{1}{|\mathcal B|}\sum_{i\in\mathcal B}\frac{1}{|V_i|}
   \sum_{t\in V_i}\operatorname{BCE}(r_{i,t},p_{i,t}^{H}),\\
 \mathcal L_{\mathrm{balance}}
 &=\mathbb E_{i,t}\!\left[\mathcal H(\widetilde{\mathbf a}_{i,t})\right]
   -\mathcal H\!\left(\mathbb E_{i,t}[\widetilde{\mathbf a}_{i,t}]\right),
 \qquad \mathcal L=\mathcal L_{\mathrm{pred}}+0.1\mathcal L_{\mathrm{balance}}.
\end{align}
The entropy averages cover non-padding steps in the batch.
Padding and label masks exclude invalid positions from the prediction loss; full trajectories are retained.
The ordinary batch iterator preserves input order, without an additional epoch-wise shuffle.
Training seeds initialize Python, NumPy, PyTorch, and CUDA random-number generators.

\textbf{Selection and restoration.}
Validation is performed after each epoch.
The checkpoint key is lexicographic: higher validation AUPRC, then lower validation expected calibration error (ECE), then higher validation alert lead.
An exact tie retains the earlier checkpoint.
The linear temperature schedule is $\tau_e=1-0.75(e-1)/23$.
Selected epochs range from 9--23 on WebArena, 16--24 on $\tau^2$-Bench, 15--24 on SkillsBench, and 14--24 on TerminalBench.
Evaluation uses the selected epoch's temperature, event projection, and GRU weights, together with the same serializer and fitted vectorizer.

\subsection{Risk-Aligned Induction and Audit Readouts}
\label{app:risk_merge_details}

\textbf{Prefix counts and merging.}
The prefix tree is built from every labeled training prefix and its hard-event sequence.
Each occurrence increments the positive and total counts $(m_s,n_s)$ only at its endpoint; repeated histories with conflicting labels remain in the same counts.
Supported states have empirical risk $\widehat r_s=m_s/n_s$.
Two nodes $u,v$ are compatible when
\begin{equation}
 |\widehat r_u-\widehat r_v|
 \leq\sqrt{\tfrac12\log(2/\eta)}
 \left(n_u^{-1/2}+n_v^{-1/2}\right),
 \qquad \eta=0.05,
 \label{eq:app_merge_rule}
\end{equation}
and the same test holds recursively for successors under every shared event.
If either node has zero support, the local test passes; recursive successor checks still apply.
No conflicting-prefix filtering is used by this induction rule.

The retained red set starts with the root.
Among its non-red children, the algorithm selects the node with greatest support, breaking ties by the smallest node index.
It merges into the compatible red node with closest empirical risk, then greatest support, then smallest index; if none is compatible, the selected node becomes red.
Folding sums endpoint counts and recursively combines shared successors.
After merging, the implementation compacts reachable nodes and completes absent transitions with an absorbing sink.

\textbf{State risks.}
Calibration replay gives a state positive count $c_s$, support $N_s$, and global positive-prefix prevalence $\bar p_{\mathrm{cal}}$.
The risk lookup is
\begin{equation}
 \rho(s)=\frac{c_s+5\bar p_{\mathrm{cal}}}{N_s+5}.
 \label{eq:app_state_risk}
\end{equation}
This rule also covers unsupported states and the sink, assigning the global prevalence when $N_s=0$.
An own-FAR10 readout applies Eq.~\ref{eq:app_far_threshold} to maxima of $\rho(s_t)$ on successful calibration trajectories and freezes the resulting scalar threshold.

\textbf{Matched alert set.}
Let $g_{i,t}=\mathbb I[r_{i,t}\geq\gamma_G]$ be the frozen GRU decision on calibration.
For each visited state, compute $u_s=\sum_{i,t:s_{i,t}=s}g_{i,t}$ and $v_s=\#\{(i,t):s_{i,t}=s\}$.
Sort states by decreasing $u_s/v_s$, breaking ties by decreasing $u_s$ and then increasing state ID.
For the first $k$ states $S_k$, choose
\[
 k^*=\arg\max_k F_{1,\mathrm{cal}}\bigl(\mathbb I[s_{i,t}\in S_k],g_{i,t}\bigr),
 \qquad a_D(\tau_{:t})=\mathbb I[s_t\in S_{k^*}].
\]
The scan retains the smallest $k$ attaining the maximum.
The readout is undefined if the GRU has no calibration alerts; states unvisited on calibration are excluded from the selected set.
State membership is then frozen for test.
This optimizes GRU-relative agreement without a FAR constraint and is distinct from thresholding $\rho(s)$.

Algorithm~\ref{alg:app_pipeline} constructs both DFA readouts after model selection and graph induction.
Its inputs are calibration records containing the replayed state $s_{i,t}$, prefix label $p_{i,t}$, and frozen GRU alert $g_{i,t}$.
Write $\operatorname{Threshold}_{\alpha}$ for the successful-trajectory maximum rule in Eq.~\ref{eq:app_far_threshold}.
The symbol $\bot$ denotes an undefined matched-alert readout, rather than an empty alert set.

\begin{algorithm}[!ht]
\caption{Calibrating DFA risk and GRU-matched alert readouts}
\label{alg:app_pipeline}
\begin{algorithmic}[1]
\Require Frozen DFA states $Q$ (including sink); nominal FAR budget $\alpha$
\Require Calibration records $\{(y_i,(s_{i,t},p_{i,t},g_{i,t})_{t=1}^{T_i})\}_{i=1}^{n}$
\Ensure State risks $\rho$, risk threshold $\gamma_D$, matched alert states $S^*$
\State $\bar p\gets\sum_{i,t}p_{i,t}/\sum_i T_i$; $G\gets\sum_{i,t}g_{i,t}$
\For{each $s\in Q$}
    \State $I_s\gets\{(i,t):s_{i,t}=s\}$; $N_s\gets|I_s|$
    \State $c_s\gets\sum_{(i,t)\in I_s}p_{i,t}$; $u_s\gets\sum_{(i,t)\in I_s}g_{i,t}$
    \State $\rho(s)\gets(c_s+5\bar p)/(N_s+5)$
\EndFor
\State $\gamma_D\gets\operatorname{Threshold}_{\alpha}(\{\max_t\rho(s_{i,t}):y_i=1\})$
\If{$G=0$}
    \State \Return $(\rho,\gamma_D,\bot)$
\EndIf
\State $(s_1,\ldots,s_m)\gets\operatorname{sort}_{s:N_s>0}(-u_s/N_s,-u_s,s)$
\State $\mathrm{TP}\gets0$; $\mathrm{FP}\gets0$; $F^*\gets-1$; $S^*\gets\emptyset$
\For{$k=1,\ldots,m$}
    \State $\mathrm{TP}\gets\mathrm{TP}+u_{s_k}$; $\mathrm{FP}\gets\mathrm{FP}+N_{s_k}-u_{s_k}$
    \State $F_k\gets 2\mathrm{TP}/(2\mathrm{TP}+\mathrm{FP}+G-\mathrm{TP})$
    \If{$F_k>F^*$}
        \State $F^*\gets F_k$; $S^*\gets\{s_1,\ldots,s_k\}$ \Comment{retain smallest tied set}
    \EndIf
\EndFor
\State \Return $(\rho,\gamma_D,S^*)$
\end{algorithmic}
\end{algorithm}

\subsection{Computational Setup}
\label{app:compute_resources}
\label{app:artifact_availability}

Training uses dense, padded GRU batches on CUDA, with step embeddings stored in 32-bit floating point and reused across epochs.
Input, model, and optimization settings are given in Tables~\ref{tab:app_frontends} and~\ref{tab:app_training_settings}; Algorithm~\ref{alg:app_pipeline} specifies audit-readout calibration after training.

Table~\ref{tab:app_compute} summarizes recorded training and induction times.
Training time includes data loading, encoding, optimization, validation, checkpoint saving, and scheduling overhead.
Induction time covers graph construction and merging, excluding encoding, data loading, readout calibration, and test scoring.

\begin{table}[!ht]
\centering\small
\caption{Elapsed-time ranges in seconds.}
\label{tab:app_compute}
\begin{tabular}{lrr}
\toprule
Benchmark  & Training (s) & DFA induction (s) \\
\midrule
WebArena  & 121--302 & 0.07--0.57 \\
$\tau^2$-Bench  & 180--301 & 0.36--1.46 \\
SkillsBench  & 242--302 & 2.82--3.49 \\
TerminalBench  & 2,722--7,665 & 9.13--16.20 \\
\bottomrule
\end{tabular}
\end{table}

%% file: appendix/data_protocol.tex
\section{Data and Evaluation Protocol}
\label{app:datasets}

\subsection{Trace Sources and Benchmark Composition}
\label{app:agent_model_counts}

The main evaluation uses frozen execution traces from WebArena, $\tau^2$-Bench, SkillsBench, and TerminalBench.
Table~\ref{tab:app_corpus_composition} describes the frozen corpora, before splitting.
An agent implementation denotes the program and execution variant that drives the underlying language model.
Model counts pool provider aliases, snapshot dates, and equivalent version spellings; agent counts retain named execution variants while pooling prompt configurations.
The union across corpora contains 17 agent implementations and 39 backbone models.

\begin{table}[!ht]
\centering\small
\caption{Composition of the current trace corpora. Success rates and mean lengths are computed over trajectories; agent and model counts use the normalization rules above.}
\label{tab:app_corpus_composition}
\begin{tabular}{lrrrrrr}
\toprule
Corpus & Trajectories & Tasks & Agents & Models & Success (\%) & Mean steps \\
\midrule
WebArena & 6,047 & 812 & 3 & 5 & 22.4 & 9.3 \\
$\tau^2$-Bench & 10,832 & 392 & 4 & 4 & 66.3 & 15.5 \\
SkillsBench & 10,951 & 89 & 5 & 7 & 27.5 & 32.4 \\
TerminalBench & 34,397 & 89 & 9 & 31 & 32.2 & 36.9 \\
\bottomrule
\end{tabular}
\end{table}

\textbf{Trace sources.}
WebArena combines 4,425 executions from its 2023 official trace releases\footnote{\url{https://github.com/web-arena-x/webarena/blob/main/resources/README.md}} with 810 ColorBrowserAgent executions\footnote{\url{https://github.com/MadeAgents/browser-agent/tree/main/trajectories/webarena}} using GPT-5 and 812 WebOperator executions\footnote{\href{https://huggingface.co/datasets/mahirlabibdihan/weboperator-experiments}{\nolinkurl{huggingface.co/datasets/mahirlabibdihan/weboperator-experiments}}} using GPT-4o.
The $\tau^2$-Bench traces come from the official bundled evaluation results.\footnote{\url{https://github.com/sierra-research/tau2-bench/tree/main/data/tau2/results/final}}
SkillsBench uses archived trial records\footnote{\url{https://github.com/benchflow-ai/skillsbench-trajectories}} containing each trial's agent trace, configuration, and verifier result.
TerminalBench uses the public trajectory dataset.\footnote{\url{https://huggingface.co/datasets/yoonholee/terminalbench-trajectories}}

\subsection{Data Splits}
\label{app:dataset_splits}

Trajectory IDs are disjoint across train, calibration, validation, and test; the allocator ensures that every task in a held-out split also occurs in train.
Not every task occurs in every held-out split: WebArena has 549, 546, and 549 distinct tasks in calibration, validation, and test, respectively; SkillsBench has 86, 86, and 88.
All 392 $\tau^2$-Bench tasks and all 89 TerminalBench tasks occur in each split.

\begin{table}[!ht]
\centering\small
\caption{Frozen split sizes. Parentheses give the number of successful trajectories. Validation selects the GRU checkpoint; calibration determines the frozen readouts.}
\label{tab:app_current_splits}
\begin{tabular}{lrrrr}
\toprule
Corpus & Train & Calibration & Validation & Test \\
\midrule
WebArena & 4,233 (949) & 605 (136) & 605 (136) & 604 (135) \\
$\tau^2$-Bench & 7,583 (5,027) & 1,083 (718) & 1,083 (718) & 1,083 (718) \\
SkillsBench & 7,666 (2,107) & 1,095 (301) & 1,095 (301) & 1,095 (301) \\
TerminalBench & 24,077 (7,728) & 3,440 (1,111) & 3,440 (1,118) & 3,440 (1,122) \\
\bottomrule
\end{tabular}
\end{table}

\textbf{Allocation.}
WebArena, $\tau^2$-Bench, and SkillsBench use task-by-outcome integer allocation with target proportions 70/10/10/10.
The allocator satisfies global outcome quotas while minimizing absolute deviations from the target group allocations.
Within groups, a hash of the split seed and trajectory ID fixes the assignment order; the split seed is 13.
The all-split task-coverage constraint applies to tasks with at least eight executions for WebArena and four for the other two corpora.
TerminalBench instead shuffles executions within each task with split seed 1, assigns rounded 10\% shares to calibration, validation, and test, and gives the remainder to train, while ensuring each task appears in each split.
Its allocation is not stratified by outcome.

\subsection{Prefix Labels and Input Availability}
\label{app:prefix_labels}

Let $y_i=1$ denote trajectory success, $T_i$ its observed length, and $t\in\{1,\ldots,T_i\}$ the step index after observing the action and its result.
For a finite warning horizon $H\in\mathbb Z^+$, the prefix target is
\begin{equation}
 p_{i,t}^{H}=\mathbb I[y_i=0\ \land\ T_i-t\leq H].
 \label{eq:labeling}
\end{equation}
A failed trajectory contributes at most $H+1$ positive prefixes, including termination.
Earlier prefixes are negative for this horizon, even if they already contain evidence of eventual failure; all prefixes of successful trajectories are negative.
For $H=\infty$, every prefix receives the trajectory's failure label $\mathbb I[y_i=0]$.
The stored implementation uses $H=100{,}000$ for this endpoint, exceeding every trajectory length in these corpora.
Writing $p_t$ for $p_{i,t}^{H}$ when the trajectory and horizon are fixed, the target risk is $\mathrm{Risk}_H(\tau_{:t})=\Pr(p_t=1\mid\tau_{:t})$.
It concerns failure of the run within the chosen horizon, rather than whether the current action is correct.
Neither future steps nor $T_i$ are supplied as an online feature; $T_i$ and the outcome define the training target and retrospective evaluation.

\begin{table}[!ht]
\centering\small
\caption{Locked-test positive-prefix prevalence $r_H$ for each evaluated horizon. Values are percentages and are identical across training seeds. Test prefix counts are 5,692 / 16,851 / 35,663 / 124,886 in column order.}
\label{tab:app_horizon_prevalence}
\begin{tabular}{rr@{\hspace{1.5em}}rr@{\hspace{1.5em}}rr@{\hspace{1.5em}}rr}
\toprule
\multicolumn{2}{c}{WebArena} & \multicolumn{2}{c}{$\tau^2$-Bench} & \multicolumn{2}{c}{SkillsBench} & \multicolumn{2}{c}{TerminalBench} \\
$H$ & $r_H$ & $H$ & $r_H$ & $H$ & $r_H$ & $H$ & $r_H$ \\
\midrule
3 & 29.83 & 3 & 8.66 & 3 & 8.74 & 3 & 7.04 \\
5 & 40.14 & 5 & 12.97 & 5 & 13.02 & 5 & 10.07 \\
9 & 54.44 & 8 & 19.35 & 8 & 19.15 & 8 & 14.16 \\
$\infty$ & 80.46 & 12 & 27.25 & 16 & 32.85 & 18 & 25.30 \\
 & & 15 & 31.98 & 24 & 43.48 & 28 & 33.82 \\
 & & $\infty$ & 43.43 & 32 & 51.27 & 37 & 39.89 \\
 & & & & $\infty$ & 76.93 & $\infty$ & 75.95 \\
\bottomrule
\end{tabular}
\end{table}

\subsection{Calibration and Metric Definitions}
\label{app:eval_protocol}

Each benchmark and horizon has a separately trained GRU.
The TF-IDF encoder and model weights are fitted using training trajectories; validation AUPRC selects the checkpoint.
The selected model is frozen before calibration.
For the DFA, training prefixes determine the graph, calibration prefixes determine state risks, and calibration GRU alerts determine the matched alert set.
Test trajectories are used only to evaluate the frozen model and readouts.

\textbf{Frozen FAR threshold.}
For the $n$ successful calibration trajectories, let $m_i=\max_t r_{i,t}$ and sort these maxima as $m_{(1)}\leq\cdots\leq m_{(n)}$.
At nominal budget $\alpha=0.10$, the implementation sets
\begin{equation}
 k=\left\lceil(n+1)(1-\alpha)\right\rceil,
 \qquad
 \gamma=\begin{cases}
 \operatorname{nextafter}(m_{(k)},+\infty),&k\leq n,\\
 +\infty,&k>n.
 \end{cases}
 \label{eq:app_far_threshold}
\end{equation}
Here \texttt{nextafter} returns the next representable value above the selected maximum, so the alert test $r\geq\gamma$ excludes calibration ties at that value.
All reported cells have successful calibration trajectories (Table~\ref{tab:app_current_splits}).
For a scorer fixed independently of calibration, the split-conformal interpretation requires exchangeability of successful calibration and future trajectories.
Under this assumption, the future successful-run maximum exceeds $m_{(k)}$ with probability at most $(n+1-k)/(n+1)\leq\alpha$ when $k\leq n$.
This follows from its exchangeable rank among $n+1$ maxima; excluding ties is conservative, and $k>n$ produces no alerts.
Thus $\Pr(A<T\mid y=1)\leq\Pr(A\leq T\mid y=1)\leq\alpha$, with probability taken over calibration and the future successful trajectory.
The bound is marginal over calibration and future trajectories.
The DFA's state risks and its own FAR threshold share calibration data; the same numerical rule therefore supplies an empirical operating point for that readout.
The matched alert set optimizes agreement with GRU alerts and is not constrained by FAR.

\textbf{Ranking and warning.}
AUPRC is computed as non-interpolated average precision over the pooled labeled prefixes.
It weights prefixes equally, whereas the prediction training loss weights trajectories equally.
Let $A_i=\min\{t:r_{i,t}\geq\gamma\}$, with $A_i=T_i+1$ for no alert, and let $N_+$ and $N_-$ count successful and failed test trajectories.
The warning metrics are
\begin{align}
 \mathrm{FAR}&=\frac{1}{N_+}\sum_{i:y_i=1}\mathbb I[A_i<T_i],\\
 \mathrm{Recall}&=\frac{1}{N_-}\sum_{i:y_i=0}\mathbb I[A_i<T_i],\\
 \mathrm{Lead}&=\frac{1}{N_-}\sum_{i:y_i=0}\frac{(T_i-A_i)_+}{T_i}.
\end{align}

%% file: appendix/warning_results.tex
\section{Extended Warning Results}
\label{app:warning_results}

\subsection{Full Horizon Results and Seed Variation}
\label{app:horizon}
Tables~\ref{tab:app_horizon_webarena}--\ref{tab:app_horizon_terminal} expand the main warning figure for all 24 benchmark--horizon settings and three training seeds per setting.
Validation AUPRC selects checkpoints; test AUPRC uses held-out prefixes.
Both prevalences are reported because changing $H$ changes the positive target and its frequency.
The deployment columns use the frozen calibration threshold and measure trajectory-level FAR, recall, and lead as defined in Appendix~\ref{app:eval_protocol}.

Each setting reports the mean and population SD across its three seeds.
The final column identifies every seed whose realized test FAR exceeds 0.10, even where the mean remains below the nominal budget.
\input{tables/appendix_warning_horizons.tex}

\subsection{Baselines on the Same Data Splits}
\label{app:baseline_horizon_protocol}
Table~\ref{tab:app_baseline_horizons} reports held-out test AUPRC for all five methods on the same benchmark splits across all 24 benchmark--horizon settings.
For consistency with the other tables, population SD is computed from the three seed scores.

\textbf{PPM-LSTM control.}
Each step is mapped to a categorical activity \texttt{tool|status}, the tool name paired with its status field (34--139 activities per benchmark, fitted on training trajectories with an \texttt{<UNK>} entry), and fed one-hot to a single-layer LSTM of width 64 trained for three epochs on the same horizon labels as PrefixGuard.

\textbf{Log-anomaly controls.}
These models train on successful training trajectories and reuse their scores when prefixes are relabeled for different horizons.
DeepLog~\citep{du2017deeplog} uses two LSTM layers of width 64, trains for 300 epochs, and scores each prefix by the mean next-key negative log-likelihood over its windows.
Its context window is ten events, reduced to three for WebArena.
LogBERT~\citep{guo2021logbert} uses four Transformer layers, hidden width 256, four attention heads, and mask probability 0.65, and scores each prefix by the mean masked-key negative log-likelihood.
It trains for at most 200 epochs with an early-stopping patience of ten.
These NLL scores replace the reference top-$g$ miss ratios, whose few distinct values leave the split-conformal threshold with no admissible alert; they remain anomaly scores rather than supervised failure probabilities.
PCA fits zero-centered TF-IDF activity-count features on successful data, retaining 95\% explained variance with at most 50,000 fitting rows, and uses squared reconstruction error.
The selected baseline recipes and scoring variants are fixed across the reported horizon comparisons.
\input{tables/appendix_baseline_horizons.tex}

\subsection{Baseline Alert Operating Points}
\label{app:baseline_deployment}
Tables~\ref{tab:baseline_activity_deployment} and~\ref{tab:baseline_logad_deployment} report recall, realized test FAR, and lead for the same baseline runs.
Each method selects its own split-conformal threshold from successful calibration trajectories at the nominal 10\% FAR budget and is evaluated on locked test.
PPM-LSTM scores depend on $H$, whereas the three log-anomaly scores do not, so their first-alert metrics are shown once per benchmark.
PCA attains 1.000 mean recall on $\tau^2$-Bench, SkillsBench, and TerminalBench, but its lead is only .042--.082, illustrating why recall and warning time must be read together.
\input{tables/appendix_baseline_deployment.tex}

\newpage
\subsection{Zero-Shot LLM Judge Reference}
\label{app:llm_judge}
Table~\ref{tab:llm_judges} reports the zero-shot DeepSeek-V4-Pro judge shown as the dashed reference in Figure~\ref{fig:horizon-tradeoff}.
The judge scores a sample of 200 test prefixes per benchmark in a single pass at the $H{=}3$ window, receiving the full StepView prefix and the task intent through a 1M-token context window, and returns an integer failure probability that is ranked for AUPRC.

\begin{table}[!ht]
\centering\small
\caption{Zero-shot DeepSeek-V4-Pro judge test AUPRC at $H{=}3$ on 200 sampled test prefixes per benchmark, single pass. PCA and PrefixGuard values are the full-test $H{=}3$ means of Table~\ref{tab:app_baseline_horizons}.}
\label{tab:llm_judges}
\begin{tabular}{lcccc}
\toprule
Scorer & WebArena & $\tau^2$-Bench & SkillsBench & TerminalBench \\
\midrule
DeepSeek-V4-Pro (200-prefix sample) & .450 & .396 & .080 & .107 \\
PCA (full test) & .440 & .427 & .246 & .122 \\
PrefixGuard (full test) & .836 & .699 & .591 & .556 \\
\bottomrule
\end{tabular}
\end{table}

The judge uses a fixed system message (Listing~\ref{lst:judge_system}) and a per-prefix user message rendered from the observed StepView prefix (Listing~\ref{lst:judge_user}).

\begin{lstlisting}[style=promptstyle,caption={LLM judge system message},label={lst:judge_system}]
You are an evaluation system for prefix-level failure warning prediction in tool-using AI agent trajectories. You only see the observed prefix, not future steps. The prompt includes WARNING_HORIZON_STEPS. Your task is to estimate the probability that the observed prefix is a positive warning prefix: the final trajectory fails and the remaining suffix length is at most WARNING_HORIZON_STEPS. This is not asking whether the trajectory eventually fails at any later time; a failing trajectory can still be negative if failure is not imminent within the warning horizon. Rate this positive warning prefix probability as an INTEGER from 0 to 100. 0 = certainly negative under the warning-label contract, 100 = certainly positive under the warning-label contract. Most prefixes are not extreme cases -- use the full range 0-100. Do NOT default to 0 or 100 unless the evidence is overwhelming. Output exactly one JSON object: {"p_fail": <integer 0-100>}
\end{lstlisting}

\begin{lstlisting}[style=promptstyle,caption={LLM judge user message template; angle-bracket fields are filled per prefix},label={lst:judge_user}]
TASK: <task_intent>
PREFIX_STEP_INDEX: <t>
PREFIX_OBSERVED_STEPS: 1..<t>
WARNING_HORIZON_STEPS: <H>
INPUT_RENDER: stepview

OBSERVED_PREFIX:
STEP <i>:
  METADATA=[...]
  ACTION=[action=<action_text>; tool=<tool_name>; args=<json>]
  RESULT=[status=<status>; text=<result_text>]
  OBSERVATION=[...]
...

Return only JSON with a calibrated positive-warning-prefix probability, for example:
{"p_fail": 37}
\end{lstlisting}

%% file: tables/appendix_warning_horizons.tex
\begin{table}[!ht]
\centering\footnotesize
\setlength{\tabcolsep}{3pt}
\caption{WebArena horizon results. $r_v/r_t$ gives validation/test positive-prefix prevalence. FAR, recall, and lead use locked test.}
\label{tab:app_horizon_webarena}
\begin{tabular}{rccccccc}
\toprule
$H$ & Val AP & Test AP & $r_v/r_t$ & FAR & Recall & Lead & Seeds $>$.10 \\
\midrule
3 & $.881\pm.003$ & $.836\pm.001$ & .299/.298 & $.074\pm.006$ & $.514\pm.026$ & $.035\pm.003$ & -- \\
5 & $.861\pm.008$ & $.825\pm.007$ & .403/.401 & $.072\pm.007$ & $.527\pm.038$ & $.077\pm.012$ & -- \\
9 & $.858\pm.005$ & $.837\pm.001$ & .545/.544 & $.074\pm.016$ & $.547\pm.046$ & $.095\pm.061$ & -- \\
$\infty$ & $.974\pm.001$ & $.953\pm.000$ & .796/.805 & $.091\pm.021$ & $.657\pm.018$ & $.401\pm.018$ & 42 \\
\bottomrule
\end{tabular}
\end{table}

\begin{table}[!ht]
\centering\footnotesize
\setlength{\tabcolsep}{3pt}
\caption{$\tau^2$-Bench horizon results. $r_v/r_t$ gives validation/test positive-prefix prevalence. FAR, recall, and lead use locked test; the last column lists seeds with test FAR $>0.10$.}
\label{tab:app_horizon_tau2}
\begin{tabular}{rccccccc}
\toprule
$H$ & Val AP & Test AP & $r_v/r_t$ & FAR & Recall & Lead & Seeds $>$.10 \\
\midrule
3 & $.735\pm.003$ & $.699\pm.014$ & .088/.087 & $.123\pm.003$ & $.997\pm.002$ & $.165\pm.006$ & 13,42,123 \\
5 & $.701\pm.005$ & $.689\pm.002$ & .131/.130 & $.131\pm.002$ & $.995\pm.002$ & $.208\pm.006$ & 13,42,123 \\
8 & $.692\pm.004$ & $.713\pm.005$ & .196/.193 & $.114\pm.006$ & $.988\pm.006$ & $.262\pm.003$ & 13,42,123 \\
12 & $.712\pm.004$ & $.744\pm.002$ & .273/.273 & $.118\pm.004$ & $.862\pm.003$ & $.344\pm.012$ & 13,42,123 \\
15 & $.738\pm.002$ & $.766\pm.002$ & .316/.320 & $.109\pm.004$ & $.858\pm.009$ & $.384\pm.012$ & 13,42,123 \\
$\infty$ & $.847\pm.002$ & $.855\pm.003$ & .407/.434 & $.108\pm.006$ & $.878\pm.013$ & $.553\pm.015$ & 13,42,123 \\
\bottomrule
\end{tabular}
\end{table}

\begin{table}[!ht]
\centering\footnotesize
\setlength{\tabcolsep}{3pt}
\caption{SkillsBench horizon results. $r_v/r_t$ gives validation/test positive-prefix prevalence. FAR, recall, and lead use locked test; the last column lists seeds with test FAR $>0.10$.}
\label{tab:app_horizon_skillsbench}
\begin{tabular}{rccccccc}
\toprule
$H$ & Val AP & Test AP & $r_v/r_t$ & FAR & Recall & Lead & Seeds $>$.10 \\
\midrule
3 & $.567\pm.006$ & $.591\pm.008$ & .087/.087 & $.090\pm.012$ & $.741\pm.080$ & $.059\pm.006$ & 123 \\
5 & $.566\pm.006$ & $.591\pm.006$ & .128/.130 & $.090\pm.012$ & $.688\pm.090$ & $.099\pm.004$ & 123 \\
8 & $.607\pm.003$ & $.628\pm.002$ & .189/.191 & $.073\pm.012$ & $.552\pm.029$ & $.127\pm.007$ & -- \\
16 & $.713\pm.001$ & $.723\pm.001$ & .325/.328 & $.087\pm.002$ & $.592\pm.005$ & $.226\pm.004$ & -- \\
24 & $.777\pm.001$ & $.780\pm.003$ & .428/.435 & $.079\pm.006$ & $.549\pm.023$ & $.272\pm.008$ & -- \\
32 & $.811\pm.003$ & $.808\pm.005$ & .505/.513 & $.097\pm.006$ & $.551\pm.021$ & $.331\pm.027$ & 13 \\
$\infty$ & $.960\pm.001$ & $.967\pm.001$ & .763/.769 & $.101\pm.010$ & $.696\pm.007$ & $.560\pm.003$ & 13,42 \\
\bottomrule
\end{tabular}
\end{table}

\begin{table}[!ht]
\centering\footnotesize
\setlength{\tabcolsep}{3pt}
\caption{TerminalBench horizon results. $r_v/r_t$ gives validation/test positive-prefix prevalence. FAR, recall, and lead use locked test; the last column lists seeds with test FAR $>0.10$.}
\label{tab:app_horizon_terminal}
\begin{tabular}{rccccccc}
\toprule
$H$ & Val AP & Test AP & $r_v/r_t$ & FAR & Recall & Lead & Seeds $>$.10 \\
\midrule
3 & $.547\pm.002$ & $.556\pm.005$ & .067/.070 & $.095\pm.004$ & $.964\pm.001$ & $.224\pm.009$ & 42 \\
5 & $.529\pm.002$ & $.540\pm.001$ & .096/.101 & $.109\pm.006$ & $.961\pm.008$ & $.258\pm.014$ & 7,42,123 \\
8 & $.535\pm.001$ & $.551\pm.001$ & .135/.142 & $.111\pm.005$ & $.939\pm.012$ & $.319\pm.013$ & 7,42,123 \\
18 & $.610\pm.001$ & $.629\pm.003$ & .241/.253 & $.102\pm.001$ & $.869\pm.006$ & $.393\pm.011$ & 7,42,123 \\
28 & $.670\pm.003$ & $.694\pm.003$ & .321/.338 & $.097\pm.004$ & $.854\pm.008$ & $.430\pm.011$ & 13 \\
37 & $.710\pm.001$ & $.733\pm.002$ & .378/.399 & $.111\pm.003$ & $.808\pm.004$ & $.477\pm.012$ & 13,42,123 \\
$\infty$ & $.976\pm.001$ & $.963\pm.002$ & .768/.760 & $.109\pm.006$ & $.879\pm.011$ & $.623\pm.007$ & 13,42,123 \\
\bottomrule
\end{tabular}
\end{table}

%% file: tables/appendix_baseline_horizons.tex
\begin{table}[!ht]
\centering\footnotesize
\setlength{\tabcolsep}{3pt}
\caption{Held-out test AUPRC on the same benchmark splits (mean $\pm$ population SD over three seeds) for all 24 benchmark--horizon settings; each horizon has its own target and prevalence. PPM-LSTM uses \texttt{tool|status} activities; DeepLog and LogBERT are scored by mean next-key and masked-key NLL. PrefixGuard values repeat the Test AP columns of Tables~\ref{tab:app_horizon_webarena}--\ref{tab:app_horizon_terminal}.}
\label{tab:app_baseline_horizons}
\begin{tabular}{rccccc}
\toprule
$H$ & PrefixGuard & PPM-LSTM & DeepLog & LogBERT & PCA \\
\midrule
\multicolumn{6}{l}{\textbf{WebArena}} \\
3 & $.836\pm.001$ & $.483\pm.007$ & $.360\pm.002$ & $.293\pm.007$ & $.440\pm.000$ \\
5 & $.825\pm.007$ & $.544\pm.006$ & $.451\pm.002$ & $.403\pm.005$ & $.526\pm.000$ \\
9 & $.837\pm.001$ & $.661\pm.000$ & $.585\pm.001$ & $.547\pm.007$ & $.654\pm.000$ \\
$\infty$ & $.953\pm.000$ & $.888\pm.004$ & $.817\pm.000$ & $.795\pm.006$ & $.850\pm.000$ \\
\multicolumn{6}{l}{\textbf{$\tau^2$-Bench}} \\
3 & $.699\pm.014$ & $.618\pm.009$ & $.303\pm.001$ & $.071\pm.002$ & $.427\pm.000$ \\
5 & $.689\pm.002$ & $.595\pm.008$ & $.342\pm.001$ & $.110\pm.002$ & $.409\pm.000$ \\
8 & $.713\pm.005$ & $.595\pm.003$ & $.386\pm.001$ & $.172\pm.003$ & $.420\pm.000$ \\
12 & $.744\pm.002$ & $.599\pm.004$ & $.425\pm.001$ & $.258\pm.004$ & $.451\pm.000$ \\
15 & $.766\pm.002$ & $.605\pm.003$ & $.447\pm.001$ & $.308\pm.005$ & $.471\pm.000$ \\
$\infty$ & $.855\pm.003$ & $.720\pm.001$ & $.534\pm.001$ & $.419\pm.008$ & $.557\pm.001$ \\
\multicolumn{6}{l}{\textbf{SkillsBench}} \\
3 & $.591\pm.008$ & $.468\pm.002$ & $.136\pm.000$ & $.076\pm.002$ & $.246\pm.001$ \\
5 & $.591\pm.006$ & $.434\pm.004$ & $.165\pm.001$ & $.110\pm.003$ & $.257\pm.001$ \\
8 & $.628\pm.002$ & $.454\pm.003$ & $.216\pm.001$ & $.162\pm.002$ & $.296\pm.000$ \\
16 & $.723\pm.001$ & $.551\pm.006$ & $.341\pm.001$ & $.286\pm.002$ & $.403\pm.000$ \\
24 & $.780\pm.003$ & $.621\pm.003$ & $.436\pm.001$ & $.396\pm.003$ & $.490\pm.000$ \\
32 & $.808\pm.005$ & $.681\pm.004$ & $.512\pm.001$ & $.477\pm.002$ & $.557\pm.000$ \\
$\infty$ & $.967\pm.001$ & $.870\pm.004$ & $.801\pm.002$ & $.773\pm.001$ & $.821\pm.000$ \\
\multicolumn{6}{l}{\textbf{TerminalBench}} \\
3 & $.556\pm.005$ & $.489\pm.001$ & $.100\pm.001$ & $.081\pm.001$ & $.122\pm.001$ \\
5 & $.540\pm.001$ & $.436\pm.001$ & $.120\pm.001$ & $.107\pm.001$ & $.139\pm.001$ \\
8 & $.551\pm.001$ & $.421\pm.002$ & $.152\pm.001$ & $.143\pm.002$ & $.170\pm.001$ \\
18 & $.629\pm.003$ & $.468\pm.001$ & $.251\pm.001$ & $.245\pm.003$ & $.267\pm.001$ \\
28 & $.694\pm.003$ & $.537\pm.001$ & $.333\pm.001$ & $.325\pm.004$ & $.346\pm.001$ \\
37 & $.733\pm.002$ & $.588\pm.002$ & $.393\pm.001$ & $.383\pm.004$ & $.405\pm.001$ \\
$\infty$ & $.963\pm.002$ & $.886\pm.001$ & $.750\pm.001$ & $.705\pm.003$ & $.797\pm.001$ \\
\bottomrule
\end{tabular}
\end{table}

%% file: tables/appendix_baseline_deployment.tex
\begin{table}[!ht]
\centering\footnotesize
\setlength{\tabcolsep}{5pt}
\caption{PPM-LSTM (\texttt{tool|status} activities) first-alert results at the nominal 10\% calibration FAR budget. Values are mean $\pm$ population SD across three seeds on locked test.}
\label{tab:baseline_activity_deployment}
\begin{tabular}{rccc}
\toprule
$H$ & Recall & Test FAR & Lead \\
\midrule
\multicolumn{4}{l}{\textbf{WebArena}} \\
3 & $.259\pm.007$ & $.086\pm.003$ & $.028\pm.002$ \\
5 & $.316\pm.024$ & $.089\pm.006$ & $.052\pm.013$ \\
9 & $.363\pm.027$ & $.094\pm.013$ & $.063\pm.013$ \\
$\infty$ & $.349\pm.036$ & $.094\pm.028$ & $.117\pm.017$ \\
\midrule
\multicolumn{4}{l}{\textbf{$\tau^2$-Bench}} \\
3 & $.892\pm.005$ & $.094\pm.004$ & $.057\pm.001$ \\
5 & $.883\pm.012$ & $.089\pm.008$ & $.089\pm.002$ \\
8 & $.823\pm.016$ & $.097\pm.002$ & $.115\pm.006$ \\
12 & $.774\pm.010$ & $.099\pm.005$ & $.125\pm.002$ \\
15 & $.760\pm.012$ & $.093\pm.005$ & $.131\pm.002$ \\
$\infty$ & $.752\pm.007$ & $.103\pm.002$ & $.162\pm.003$ \\
\midrule
\multicolumn{4}{l}{\textbf{SkillsBench}} \\
3 & $.984\pm.014$ & $.096\pm.000$ & $.022\pm.001$ \\
5 & $.991\pm.002$ & $.097\pm.010$ & $.032\pm.004$ \\
8 & $.990\pm.002$ & $.095\pm.004$ & $.037\pm.003$ \\
16 & $.992\pm.001$ & $.101\pm.010$ & $.048\pm.004$ \\
24 & $.990\pm.004$ & $.100\pm.008$ & $.052\pm.006$ \\
32 & $.986\pm.007$ & $.104\pm.009$ & $.051\pm.010$ \\
$\infty$ & $.903\pm.109$ & $.125\pm.004$ & $.122\pm.007$ \\
\midrule
\multicolumn{4}{l}{\textbf{TerminalBench}} \\
3 & $.989\pm.001$ & $.086\pm.005$ & $.014\pm.001$ \\
5 & $.990\pm.001$ & $.086\pm.003$ & $.020\pm.005$ \\
8 & $.991\pm.000$ & $.087\pm.006$ & $.022\pm.002$ \\
18 & $.992\pm.000$ & $.082\pm.010$ & $.026\pm.017$ \\
28 & $.993\pm.000$ & $.092\pm.001$ & $.013\pm.003$ \\
37 & $.994\pm.001$ & $.094\pm.005$ & $.015\pm.001$ \\
$\infty$ & $.994\pm.001$ & $.096\pm.005$ & $.133\pm.003$ \\
\bottomrule
\end{tabular}
\end{table}

\begin{table}[!ht]
\centering\footnotesize
\setlength{\tabcolsep}{5pt}
\caption{Log-anomaly baseline first-alert results at the same nominal 10\% calibration FAR budget. DeepLog and LogBERT are thresholded on their mean NLL scores and PCA on its SPE residual; all three are horizon-independent, so recall, test FAR, and lead are identical across horizons within each benchmark. Values are mean $\pm$ population SD across three seeds on locked test.}
\label{tab:baseline_logad_deployment}
\begin{tabular}{llccc}
\toprule
Benchmark & Method & Recall & Test FAR & Lead \\
\midrule
WebArena & DeepLog & $.183\pm.007$ & $.074\pm.000$ & $.052\pm.001$ \\
 & LogBERT & $.082\pm.011$ & $.064\pm.007$ & $.055\pm.009$ \\
 & PCA & $.405\pm.000$ & $.111\pm.000$ & $.029\pm.000$ \\
\midrule
$\tau^2$-Bench & DeepLog & $.559\pm.031$ & $.120\pm.011$ & $.075\pm.007$ \\
 & LogBERT & $.105\pm.013$ & $.097\pm.016$ & $.076\pm.009$ \\
 & PCA & $1.000\pm.000$ & $.087\pm.001$ & $.082\pm.001$ \\
\midrule
SkillsBench & DeepLog & $.280\pm.002$ & $.103\pm.005$ & $.079\pm.002$ \\
 & LogBERT & $.154\pm.011$ & $.100\pm.022$ & $.115\pm.003$ \\
 & PCA & $1.000\pm.000$ & $.099\pm.002$ & $.063\pm.000$ \\
\midrule
TerminalBench & DeepLog & $.207\pm.003$ & $.080\pm.002$ & $.045\pm.002$ \\
 & LogBERT & $.105\pm.017$ & $.097\pm.009$ & $.059\pm.010$ \\
 & PCA & $1.000\pm.000$ & $.069\pm.002$ & $.042\pm.000$ \\
\bottomrule
\end{tabular}
\end{table}

%% file: appendix/design_analysis.tex
\section{Design Analysis}

\label{sec:diagnostics}
\label{sec:rq2}
\label{app:ablations}

\subsection{Event Representation and Predictor Ablations}
\label{sec:event_interface}
\label{app:event_interface_controls}
We examine two design questions: whether prediction benefits from event history and what the event interface costs in predictive accuracy.
We evaluate four variants on the same data splits, all 24 benchmark--horizon settings, and three seeds per setting (288 runs); the joint model is retrained alongside the other variants rather than reusing main-grid checkpoints, so all paired differences are within-study.
All use the same StepView inputs, explicit data splits, trajectory-averaged prefix loss, input-order minibatches, 24-epoch training budget, and validation-AUPRC checkpoint selection.
The variants compare a jointly trained event GRU with a GRU on continuous TF-IDF features.
Two further variants use the joint model's fixed encoder, event projection, and checkpoint temperature to train a fresh GRU or a current-event MLP.
The predictor width is matched within each benchmark; input-layer and total parameter counts differ across variants.
The current-event MLP has no cross-event state, although the current StepView can contain prior context.
A companion control study on the same splits and seeds trains a pooled StepView MLP that averages the TF-IDF vectors of the observed prefix and scores it without sequential state; Table~\ref{tab:analysis_controls} compares it with the main-grid monitor.
Table~\ref{tab:analysis_controls} summarizes the two event-representation comparisons; Table~\ref{tab:component_horizons} gives absolute held-out test AUPRC and seed variability for the four variants and the pooled MLP at every horizon, including terminal prefixes.

\textbf{Sequential state helped at short horizons.}
Against the pooled StepView MLP, the event GRU had higher mean AUPRC in 9/24 settings: all four WebArena horizons, three of seven on SkillsBench, two of seven on TerminalBench, and none on $\tau^2$-Bench.
Averaged over horizons, the differences were $+2.44$, $-2.37$, $+0.51$, and $-1.78$ points on WebArena, $\tau^2$-Bench, SkillsBench, and TerminalBench.
The gains concentrate at $H{\leq}5$, reaching $+6.0$ and $+6.1$ points at $H{=}3$ on WebArena and SkillsBench, and vanish or reverse at $H{\geq}16$, where the pooled model matches or exceeds the GRU.

\textbf{Event history added information beyond the current event.}
To compare predictors on identical inputs, we freeze the event representation and train GRU and MLP heads.
The MLP sees only the current event probabilities, whereas the GRU maintains a state across events.
The GRU achieved higher mean AUPRC in 21/24 settings, including every horizon on $\tau^2$-Bench and TerminalBench.
Gains were smaller on WebArena, while SkillsBench shifted from an MLP advantage at short horizons to a GRU advantage at longer horizons.
This comparison measures the value of event history under the tested recipes, with events learned by a GRU and StepView records that can carry prior context.

\textbf{Discrete events supported finite-state replay with a modest predictive cost.}
We compare the event-based GRU with a GRU trained directly on the same TF-IDF features, without the event projection.
Averaged over horizons and seeds, the event-based model's test AUPRC was 0.36--2.38 percentage points lower across the four benchmarks (Table~\ref{tab:analysis_controls}).
The learned discrete histories enable DFA replay and inspection while retaining most of the direct-input GRU's ranking performance.
\input{tables/section51_component_horizons.tex}

\subsection{The Choice of DFA Induction Method}
\label{sec:induction_choice}
\label{app:induction_comparison}
We compare risk-compatible merging with RPNI~\citep{oncina1992rpni,vonberg2025extending} and an unmerged prefix tree (PTA) on the same frozen events, splits, and calibration readouts (Figure~\ref{fig:dfa-induction-horizons}).
The comparison contains 84 runs per method across four benchmarks and all horizons.
PTA and RPNI resolve training-prefix labels by majority vote with ties assigned to zero, whereas risk-merge retains empirical risk counts.
All methods fit state risks on calibration outcomes; their alert-state sets maximize calibration F1 relative to the GRU's frozen FAR-10 alerts and are evaluated on test.
This fidelity F1 measures alert agreement, separately from outcome-label AUPRC.

\begin{figure}[!ht]
  \centering
  \includegraphics[width=\linewidth]{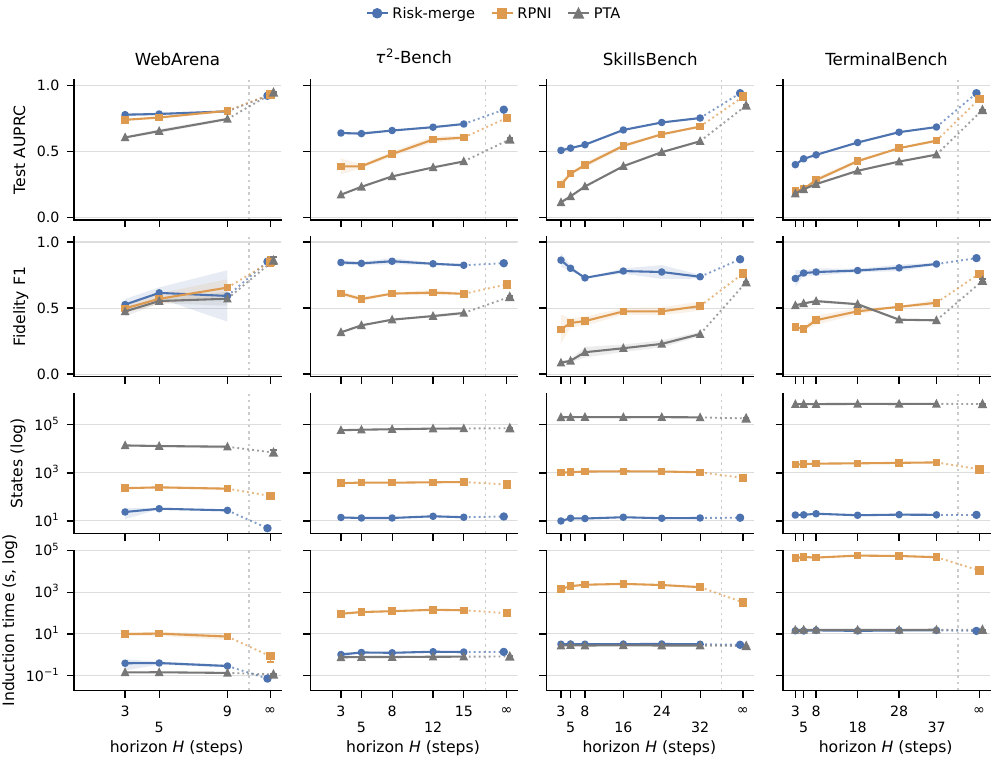}
  \caption{\textbf{DFA induction across all benchmarks and horizons.}
    Columns denote benchmarks; rows show held-out test AUPRC, matched-alert-set fidelity F1, state count, and induction wall time in seconds (last two on log scales).
    Points at $H=\infty$ have separate error bars and dotted guides from the last finite horizon.
    All four benchmarks use the evaluation splits in Appendix~\ref{app:dataset_splits}; risk-merge uses $\eta=0.05$.}
  \label{fig:dfa-induction-horizons}
\end{figure}

RPNI compressed the PTA but sometimes lowered fidelity, notably on TerminalBench, whereas risk-merge produced smaller graphs in less induction time at every horizon and higher mean AUPRC and fidelity F1 throughout $\tau^2$-Bench, SkillsBench, and TerminalBench; WebArena differences were smaller and varied with horizon.
Graph compression alone therefore does not ensure faithful replay; the merge criterion matters.
RPNI uses AALpy's Python GSM implementation.
Induction time is one measured wall-clock duration per checkpoint, summarized by mean and population SD across seeds.
For PTA/RPNI it includes majority-label resolution and graph construction; for risk-merge it includes empirical-risk graph construction and merging, measured at 0.01-second precision.
It excludes event encoding, state-risk fitting, alert-set selection, and test scoring.

The end-to-end DFA comparison follows in Appendix~\ref{app:direct_dfa}.
\input{appendix/direct_dfa.tex}

\subsection{Where Finite-State Fidelity Breaks}
\label{sec:fidelity_breaks}

Across all four benchmarks and horizons, DFAs had lower mean AUPRC than their paired GRUs and only partially reproduced GRU alerts (Figure~\ref{fig:dfa-audit}).
The directly learned DFAs also showed a soft--hard gap and silent threshold settings (Appendix~\ref{app:direct_dfa}), so these limits are not specific to risk-compatible merging.

\textbf{Fixed-checkpoint fidelity diagnostics.}
\label{app:fixed_checkpoint_controls}
\input{tables/dfa_fixed_checkpoint_controls.tex}
These controls perturb inference only.
For each of the 24 frozen SkillsBench within-task checkpoints (Section~\ref{sec:setup}), the same GRU scored either the event probabilities $\mathbf a_t$ or the one-hot event $z_t$ of Eq.~\eqref{eq:events}; parameters were checked to be identical before and after, and the hard event sequences were hash-checked across all inductions.
Each variant selected its risk threshold at the nominal 10\% false-alarm budget on the shared calibration pool; the DFA rows share calibration data between risk fitting and thresholding (Appendix~\ref{app:eval_protocol}).
Hardening reduced AUPRC by at most .035 ($H{=}3$) and reduced recall by up to .09; merging reduced AUPRC by a further .03 to .07 at every horizon and, at $H{\geq}24$, left no admissible alert state in any seed.
In the silent settings, the highest-risk state was visited by 13--18\% of successful calibration trajectories on average (11--24\% per seed) and had substantial support (about 4{,}000 calibration prefixes at $H{=}32$, seed 13), so the silence reflects a scalar threshold over a coarse risk partition rather than unsupported states: sink and abstention shares stayed below .2\% of test prefixes.
Calibration trajectory maxima took only 9--13 distinct DFA values against roughly a thousand distinct GRU values, which limits how finely a risk threshold can trade FAR against recall.
These controls use the 24 SkillsBench checkpoints; Appendix~\ref{app:dfa_own_far10} traces the same mechanism on WebArena.

\subsection{Merge-Parameter Sensitivity}
\label{app:alpha_sensitivity}
\input{figures/main/fig_dfa_alpha.tex}
\input{tables/dfa_alpha_sensitivity.tex}
Figure~\ref{fig:dfa-alpha} and Table~\ref{tab:alpha_sensitivity} vary $\eta$ on the same frozen sequences with state risks refit on the shared calibration pool for each graph.
Tightening the tolerance increased the state count from 9--14 to 80--167 across horizons while changing AUPRC by less than .03 in either direction; the intermediate settings sometimes ranked best.
More states mainly changed the attainable operating points: the number of silent seeds fell from 13 of 24 at $\eta{=}0.05$ to 1 at $\eta{=}0.5$, and recall at the budgeted risk threshold rose accordingly.
The main results use $\eta{=}0.05$; the sweep is within-method sensitivity, and no setting was chosen on test.
\input{tables/dfa_alpha_fidelity.tex}
Table~\ref{tab:alpha_fidelity} applies the matched-alert-set readout of Section~\ref{sec:finite_state_audit} to the same graphs; the $\eta{=}0.05$ column reproduces the SkillsBench column of Figure~\ref{fig:dfa-audit} run by run.
The calibration-selected alert set stayed small at every tolerance: one to two states at $\eta{=}0.05$ and at most sixteen at $\eta{=}0.5$.
Recall of GRU alert prefixes was high throughout and rose slightly with finer merging, whereas precision fell, so F1 was flat up to $\eta{=}0.2$ and lower at $\eta{=}0.5$.
Trajectory-level capture of GRU-alerted runs exceeded .8 in every cell, and the DFA's first alert preceded the GRU's by more steps as the graphs grew.
Finer merging therefore recovered admissible risk thresholds without making the DFA a closer surrogate of the GRU's decisions.

%% file: tables/section51_component_horizons.tex
\begin{table}[!ht]
\centering\scriptsize
\setlength{\tabcolsep}{3pt}
\caption{\textbf{Component ablations across all horizons.} Held-out test AUPRC. Joint: event projection and GRU trained together; continuous: GRU on TF-IDF features; pooled MLP: non-sequential MLP on the mean TF-IDF vector of the observed prefix, from a companion control study on the same splits and seeds. Refit GRU and event MLP use the same fixed representation learned by the joint model. Each horizon defines a separate prediction target. All six variants, including Joint, are retrained within this study; the retrained Joint reproduces the main-grid Test AP of Tables~\ref{tab:app_horizon_webarena}--\ref{tab:app_horizon_terminal} on three benchmarks and differs by at most .004 on TerminalBench.}
\label{tab:component_horizons}
\begin{tabular}{@{}crrrrr@{}}
\toprule
$H$ & Joint & Continuous & \shortstack{Pooled\\MLP} & \shortstack{Refit\\GRU} & \shortstack{Event\\MLP} \\
\midrule
\multicolumn{6}{c}{WebArena} \\
3 & $0.836 \pm 0.001$ & $0.829 \pm 0.008$ & $0.776 \pm 0.002$ & $0.838 \pm 0.005$ & $0.841 \pm 0.002$ \\
5 & $0.825 \pm 0.007$ & $0.830 \pm 0.001$ & $0.796 \pm 0.002$ & $0.828 \pm 0.007$ & $0.834 \pm 0.008$ \\
9 & $0.837 \pm 0.001$ & $0.843 \pm 0.003$ & $0.831 \pm 0.003$ & $0.840 \pm 0.001$ & $0.838 \pm 0.001$ \\
 $\infty$ & $0.953 \pm 0.000$ & $0.964 \pm 0.000$ & $0.951 \pm 0.001$ & $0.954 \pm 0.001$ & $0.953 \pm 0.004$ \\
\midrule
\multicolumn{6}{c}{$\tau^2$-Bench} \\
3 & $0.699 \pm 0.014$ & $0.760 \pm 0.003$ & $0.745 \pm 0.002$ & $0.690 \pm 0.006$ & $0.656 \pm 0.018$ \\
5 & $0.689 \pm 0.002$ & $0.734 \pm 0.001$ & $0.735 \pm 0.002$ & $0.687 \pm 0.001$ & $0.614 \pm 0.009$ \\
8 & $0.713 \pm 0.005$ & $0.736 \pm 0.003$ & $0.742 \pm 0.002$ & $0.716 \pm 0.006$ & $0.627 \pm 0.013$ \\
12 & $0.744 \pm 0.002$ & $0.753 \pm 0.006$ & $0.757 \pm 0.003$ & $0.743 \pm 0.004$ & $0.644 \pm 0.004$ \\
15 & $0.766 \pm 0.002$ & $0.769 \pm 0.003$ & $0.773 \pm 0.003$ & $0.769 \pm 0.003$ & $0.678 \pm 0.001$ \\
 $\infty$ & $0.855 \pm 0.003$ & $0.857 \pm 0.007$ & $0.856 \pm 0.003$ & $0.863 \pm 0.003$ & $0.799 \pm 0.002$ \\
\midrule
\multicolumn{6}{c}{SkillsBench} \\
3 & $0.591 \pm 0.008$ & $0.625 \pm 0.005$ & $0.530 \pm 0.002$ & $0.591 \pm 0.012$ & $0.604 \pm 0.003$ \\
5 & $0.591 \pm 0.006$ & $0.622 \pm 0.001$ & $0.568 \pm 0.004$ & $0.594 \pm 0.006$ & $0.578 \pm 0.003$ \\
8 & $0.628 \pm 0.002$ & $0.648 \pm 0.002$ & $0.622 \pm 0.003$ & $0.628 \pm 0.001$ & $0.581 \pm 0.005$ \\
16 & $0.723 \pm 0.001$ & $0.718 \pm 0.002$ & $0.729 \pm 0.001$ & $0.723 \pm 0.003$ & $0.656 \pm 0.004$ \\
24 & $0.780 \pm 0.003$ & $0.784 \pm 0.008$ & $0.795 \pm 0.001$ & $0.781 \pm 0.004$ & $0.708 \pm 0.005$ \\
32 & $0.808 \pm 0.005$ & $0.822 \pm 0.002$ & $0.837 \pm 0.003$ & $0.810 \pm 0.005$ & $0.759 \pm 0.004$ \\
 $\infty$ & $0.967 \pm 0.001$ & $0.967 \pm 0.000$ & $0.972 \pm 0.001$ & $0.966 \pm 0.002$ & $0.956 \pm 0.001$ \\
\midrule
\multicolumn{6}{c}{TerminalBench} \\
3 & $0.559 \pm 0.003$ & $0.580 \pm 0.001$ & $0.547 \pm 0.006$ & $0.558 \pm 0.003$ & $0.481 \pm 0.007$ \\
5 & $0.541 \pm 0.003$ & $0.554 \pm 0.004$ & $0.549 \pm 0.005$ & $0.544 \pm 0.005$ & $0.500 \pm 0.003$ \\
8 & $0.551 \pm 0.003$ & $0.562 \pm 0.001$ & $0.575 \pm 0.006$ & $0.553 \pm 0.002$ & $0.511 \pm 0.002$ \\
18 & $0.629 \pm 0.003$ & $0.647 \pm 0.002$ & $0.669 \pm 0.003$ & $0.631 \pm 0.003$ & $0.591 \pm 0.002$ \\
28 & $0.691 \pm 0.002$ & $0.702 \pm 0.004$ & $0.726 \pm 0.002$ & $0.693 \pm 0.002$ & $0.660 \pm 0.005$ \\
37 & $0.736 \pm 0.002$ & $0.751 \pm 0.001$ & $0.762 \pm 0.002$ & $0.737 \pm 0.003$ & $0.700 \pm 0.002$ \\
 $\infty$ & $0.963 \pm 0.000$ & $0.962 \pm 0.001$ & $0.963 \pm 0.000$ & $0.963 \pm 0.001$ & $0.940 \pm 0.002$ \\
\bottomrule
\end{tabular}
\end{table}

%% file: appendix/direct_dfa.tex
\subsection{Directly Learned DFAs}
\label{app:direct_dfa}

We compare risk-compatible merging with a DFA learned directly from step features, without the two-stage event-then-merge construction.

\textbf{Model and training.}
The direct model keeps the StepView features and event projection of Section~\ref{sec:monitor} but replaces the GRU with a probabilistic automaton in the style of DeepDFA~\citep{umili2024deepdfa}: a state distribution over at most $Q$ states, updated by a learned transition tensor conditioned on the relaxed event, and a per-state risk head trained with the same prefix-failure loss.
Symbol and transition relaxations are annealed linearly from temperature $1$ to their final values during training.
After training, events and transitions are hardened by argmax, unreachable states are dropped, and state risks are refit on the calibration split with the same smoothing used for the risk-merged DFA; the hardened graph is then read out exactly as in Section~\ref{sec:finite_state_audit}.
We use $H{=}3$, the same three training seeds as the paired models, and $Q\in\{8,16,32,64\}$ (48 runs).
Learning rate and final transition temperature were selected per benchmark from a $2\times3$ grid at $Q{=}16$ using validation AUPRC only (24 pilot runs).
Each model trains for 200 epochs and selects the earliest epoch with the highest \emph{hard} validation AUPRC, so checkpoint selection already favors graphs that survive hardening.

\input{tables/direct_dfa_auprc.tex}

\textbf{Prediction.}
Table~\ref{tab:direct_dfa_auprc} compares held-out test AUPRC and reachable states with the risk-merged DFAs of Figure~\ref{fig:dfa-audit}.
On SkillsBench, $\tau^2$-Bench, and TerminalBench, every direct configuration scored below the risk-merged DFA of the same seed (12 of 12 seed--$Q$ pairs per benchmark), by .04 to .10 AUPRC at the best $Q$, while the risk-merged graphs used 10 to 18 states.
On WebArena the direct DFA was slightly higher at $Q\geq16$ (about .02), consistent with the small induction-method differences on that benchmark (Appendix~\ref{app:induction_comparison}).
Increasing $Q$ did not close the gap: SkillsBench and WebArena were flat from $Q{=}8$ to $64$, $\tau^2$-Bench improved only from $Q{=}8$ to $16$, and TerminalBench gained .006 from $Q{=}32$ to $64$ while its seed spread widened tenfold.
All 48 runs also show a soft--hard gap: the soft automaton at the selected checkpoint outperformed its hardened graph by .04 to .11 AUPRC (Table~\ref{tab:direct_dfa_alerts}), so end-to-end training did not remove the hardening loss diagnosed in Appendix~\ref{app:fixed_checkpoint_controls}.
On TerminalBench all 12 best checkpoints occurred within the first 48 epochs (median 24.5), and epoch-200 validation AUPRC was about .07 below the selected value, indicating degradation of the hardened graph late in training rather than under-training.

\input{tables/direct_dfa_alerts.tex}

\textbf{Alerts at the false-alarm budget.}
Table~\ref{tab:direct_dfa_alerts} applies the calibration-selected risk threshold at the nominal 10\% false-alarm budget.
Scores were compared in float64; the original float32 comparison absorbed the one-ulp increment of the conformal threshold and admitted tied calibration scores, which inflated FAR and recall for several runs.
Of the 48 runs, 25 raised no alert on any test trajectory, including all 12 WebArena runs, matching the silent WebArena graphs of the risk-merged DFA (Appendix~\ref{app:dfa_own_far10}).
As in Appendix~\ref{app:fixed_checkpoint_controls}, the cause is a coarse risk partition: successful calibration trajectories reached the highest-risk state often enough that any admissible threshold excluded it.
$\tau^2$-Bench is the clearest usable case (at $Q{=}16$, test FAR .088 and recall .837); recall reached at most .247 on SkillsBench and .313 on TerminalBench, with one silent seed in each TerminalBench cell at $Q\geq16$.
The budget is a calibration target, and finite test sets can exceed it ($\tau^2$-Bench $Q{=}64$: .112).

Each configuration uses three seeds; the full $Q$ grid is reported and no configuration was selected on test.

%% file: tables/direct_dfa_auprc.tex

\begin{table}[!ht]
\centering\footnotesize
\caption{Directly learned DFAs versus the risk-merged DFA at $H{=}3$ (mean $\pm$ population SD over three training seeds). AUPRC is held-out test prefix AUPRC of the hardened graph; States counts reachable states after hardening. PrefixGuard values are the paired DFAs of Figure~\ref{fig:dfa-audit}.}
\label{tab:direct_dfa_auprc}
\setlength{\tabcolsep}{4pt}
\begin{tabular}{@{}ll c cccc@{}}
\toprule
& & PrefixGuard & \multicolumn{4}{c}{Direct DFA, state budget $Q$} \\
\cmidrule(lr){4-7}
Benchmark & & DFA & 8 & 16 & 32 & 64 \\
\midrule
SkillsBench & AUPRC & $.506\pm.022$ & $.417\pm.004$ & $.425\pm.013$ & $.414\pm.014$ & $.426\pm.008$ \\
 & States & 10.0 & 8.0 & 16.0 & 31.3 & 60.3 \\
$\tau^2$-Bench & AUPRC & $.638\pm.006$ & $.487\pm.012$ & $.538\pm.019$ & $.537\pm.033$ & $.535\pm.033$ \\
 & States & 14.0 & 8.0 & 16.0 & 32.0 & 63.7 \\
TerminalBench & AUPRC & $.398\pm.004$ & $.318\pm.021$ & $.302\pm.025$ & $.353\pm.000$ & $.359\pm.031$ \\
 & States & 17.7 & 8.0 & 16.0 & 30.7 & 62.7 \\
WebArena & AUPRC & $.776\pm.011$ & $.776\pm.011$ & $.794\pm.003$ & $.794\pm.009$ & $.798\pm.014$ \\
 & States & 23.7 & 8.0 & 16.0 & 32.0 & 62.3 \\
\bottomrule
\end{tabular}
\end{table}

%% file: tables/direct_dfa_alerts.tex

\begin{table}[!ht]
\centering\footnotesize
\caption{Direct DFA operating points at the nominal 10\% false-alarm budget ($H{=}3$; three-seed means). Thresholds were selected on the calibration split and compared in float64. Recall counts failed runs alerted at or before their final step. Silent: seeds whose threshold exceeded every state risk. Soft$-$hard: AUPRC of the soft model minus its hardened DFA at the same selected checkpoint.}
\label{tab:direct_dfa_alerts}
\setlength{\tabcolsep}{4pt}
\begin{tabular}{@{}lrcccc@{}}
\toprule
Benchmark & $Q$ & Test FAR & Recall & Silent & Soft$-$hard \\
\midrule
SkillsBench & 8 & .000 & .000 & 3/3 & .106 \\
 & 16 & .000 & .000 & 3/3 & .093 \\
 & 32 & .053 & .186 & 1/3 & .104 \\
 & 64 & .060 & .247 & 0/3 & .091 \\
\midrule
$\tau^2$-Bench & 8 & .086 & .838 & 0/3 & .086 \\
 & 16 & .088 & .837 & 0/3 & .041 \\
 & 32 & .090 & .837 & 0/3 & .040 \\
 & 64 & .112 & .853 & 0/3 & .073 \\
\midrule
TerminalBench & 8 & .000 & .000 & 3/3 & .089 \\
 & 16 & .068 & .298 & 1/3 & .104 \\
 & 32 & .063 & .313 & 1/3 & .100 \\
 & 64 & .069 & .301 & 1/3 & .069 \\
\midrule
WebArena & 8 & .000 & .000 & 3/3 & .071 \\
 & 16 & .000 & .000 & 3/3 & .049 \\
 & 32 & .000 & .000 & 3/3 & .056 \\
 & 64 & .000 & .000 & 3/3 & .056 \\
\bottomrule
\end{tabular}
\end{table}

%% file: tables/dfa_fixed_checkpoint_controls.tex
\begin{table}[!ht]
\centering\footnotesize
\caption{Fixed-model controls on SkillsBench (24 trained GRUs). PrefixGuard scores event probabilities $\mathbf a_t$; one-hot input feeds the same GRU the event $z_t$ of Eq.~\eqref{eq:events}; the risk-merged DFA ($\eta=0.05$) uses the same hard events. Held-out test prefix AUPRC and recall use each model's calibration-selected threshold at the nominal 10\% false-alarm budget. Silent counts seeds whose DFA threshold exceeded every state risk.}
\label{tab:fixed_checkpoint_controls}
\setlength{\tabcolsep}{3pt}
\begin{tabular}{@{}r ccc ccc c@{}}
\toprule
& \multicolumn{3}{c}{Test AUPRC} & \multicolumn{3}{c}{Recall at 10\% FAR budget} & \\
\cmidrule(lr){2-4}\cmidrule(lr){5-7}
$H$ & PrefixGuard & w/ one-hot input & DFA & PrefixGuard & w/ one-hot input & DFA & Silent \\
\midrule
3 & $.591\pm.008$ & $.556\pm.019$ & $.506\pm.022$ & $.741\pm.080$ & $.647\pm.175$ & $.387\pm.348$ & 1/3 \\
5 & $.591\pm.006$ & $.582\pm.006$ & $.523\pm.004$ & $.688\pm.090$ & $.643\pm.125$ & $.539\pm.090$ & 0/3 \\
8 & $.628\pm.002$ & $.622\pm.004$ & $.549\pm.008$ & $.552\pm.029$ & $.523\pm.022$ & $.164\pm.232$ & 2/3 \\
16 & $.723\pm.001$ & $.719\pm.000$ & $.661\pm.005$ & $.592\pm.005$ & $.574\pm.013$ & $.223\pm.243$ & 1/3 \\
24 & $.780\pm.003$ & $.778\pm.004$ & $.718\pm.005$ & $.549\pm.023$ & $.537\pm.030$ & $.000\pm.000$ & 3/3 \\
32 & $.808\pm.005$ & $.805\pm.004$ & $.751\pm.007$ & $.551\pm.021$ & $.519\pm.029$ & $.000\pm.000$ & 3/3 \\
$\infty$ & $.967\pm.001$ & $.966\pm.002$ & $.939\pm.008$ & $.696\pm.007$ & $.703\pm.004$ & $.000\pm.000$ & 3/3 \\
\bottomrule
\end{tabular}
\end{table}

%% file: figures/main/fig_dfa_alpha.tex
\begin{figure}[h!]
  \centering
  \includegraphics[width=\linewidth]{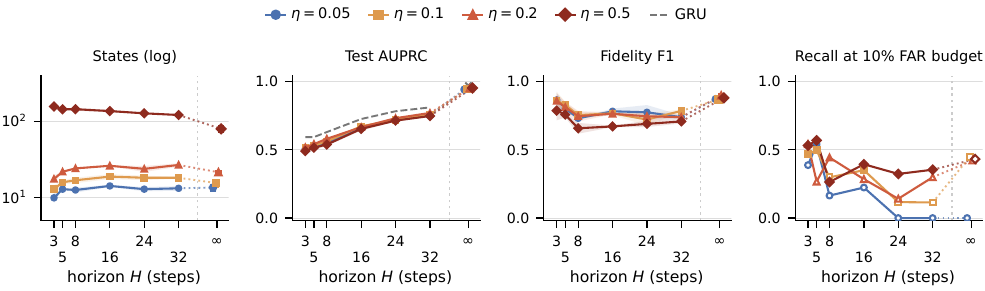}
  \caption{\textbf{Merge-tolerance sweep on SkillsBench.}
    Frozen checkpoints and events; larger $\eta$ yields more states.
    Fidelity F1 measures GRU-alert agreement; recall uses the DFA's own 10\% FAR budget.
    Hollow markers indicate seeds without feasible thresholds, counted as zero recall.}
  \label{fig:dfa-alpha}
\end{figure}

%% file: tables/dfa_alpha_sensitivity.tex
\begin{table}[!ht]
\centering\footnotesize
\caption{Merge-parameter sensitivity on SkillsBench using the same trained models and hard events. For each $\eta$: mean state count, held-out test AUPRC, and recall at the DFA's calibration-selected risk threshold under the nominal 10\% false-alarm budget. Seeds without an admissible alert state contribute zero recall; their count is in parentheses. Larger $\eta$ tightens compatibility and yields more states. No value was selected on test; the reported model uses $\eta=0.05$.}
\label{tab:alpha_sensitivity}
\setlength{\tabcolsep}{2.5pt}
\begin{tabular}{@{}r ccc ccc ccc ccc@{}}
\toprule
& \multicolumn{3}{c}{$\eta=0.05$} & \multicolumn{3}{c}{$\eta=0.1$} & \multicolumn{3}{c}{$\eta=0.2$} & \multicolumn{3}{c}{$\eta=0.5$} \\
\cmidrule(lr){2-4}\cmidrule(lr){5-7}\cmidrule(lr){8-10}\cmidrule(lr){11-13}
$H$ & States & AUPRC & Recall & States & AUPRC & Recall & States & AUPRC & Recall & States & AUPRC & Recall \\
\midrule
3 & 10.0 & .506 & .387 (1) & 13.0 & .511 & .469 & 17.7 & .517 & .540 & 157.7 & .490 & .530 \\
5 & 13.0 & .523 & .539 & 16.0 & .530 & .495 & 22.0 & .540 & .267 (2) & 145.0 & .514 & .567 \\
8 & 12.7 & .549 & .164 (2) & 17.0 & .560 & .297 (1) & 24.3 & .579 & .442 & 145.3 & .537 & .263 \\
16 & 14.3 & .661 & .223 (1) & 19.0 & .670 & .350 & 26.3 & .664 & .284 (1) & 137.3 & .651 & .393 \\
24 & 13.0 & .718 & .000 (3) & 18.3 & .717 & .117 (2) & 24.0 & .729 & .141 (1) & 128.3 & .712 & .322 \\
32 & 13.3 & .751 & .000 (3) & 18.3 & .757 & .114 (1) & 27.0 & .768 & .298 (1) & 122.0 & .745 & .353 \\
$\infty$ & 13.7 & .939 & .000 (3) & 15.7 & .942 & .447 (1) & 21.7 & .948 & .421 (1) & 80.3 & .950 & .431 (1) \\
\bottomrule
\end{tabular}
\end{table}

%% file: tables/dfa_alpha_fidelity.tex
\begin{table}[!ht]
\centering\footnotesize
\caption{Matched-alert-set fidelity across $\eta$ on SkillsBench (same trained models and hard events as Table~\ref{tab:alpha_sensitivity}). For each $\eta$: held-out test prefix precision and recall of DFA alerts relative to the GRU, their F1, and the number of selected alert states $|S_{k^\star}|$. The $\eta=0.05$ column equals the SkillsBench column of Figure~\ref{fig:dfa-audit}.}
\label{tab:alpha_fidelity}
\setlength{\tabcolsep}{2.2pt}
\begin{tabular}{@{}r cccc cccc cccc cccc@{}}
\toprule
& \multicolumn{4}{c}{$\eta=0.05$} & \multicolumn{4}{c}{$\eta=0.1$} & \multicolumn{4}{c}{$\eta=0.2$} & \multicolumn{4}{c}{$\eta=0.5$} \\
\cmidrule(lr){2-5}\cmidrule(lr){6-9}\cmidrule(lr){10-13}\cmidrule(lr){14-17}
$H$ & Prec. & Rec. & F1 & $|S_{k^\star}|$ & Prec. & Rec. & F1 & $|S_{k^\star}|$ & Prec. & Rec. & F1 & $|S_{k^\star}|$ & Prec. & Rec. & F1 & $|S_{k^\star}|$ \\
\midrule
3 & .823 & .912 & .862 & 1.7 & .791 & .952 & .859 & 2.0 & .813 & .911 & .855 & 2.7 & .729 & .867 & .785 & 6.7 \\
5 & .876 & .750 & .801 & 1.7 & .791 & .872 & .827 & 2.3 & .820 & .807 & .810 & 2.3 & .730 & .792 & .755 & 7.0 \\
8 & .644 & .854 & .728 & 1.7 & .679 & .843 & .752 & 2.3 & .698 & .799 & .744 & 2.7 & .602 & .721 & .656 & 5.7 \\
16 & .812 & .752 & .780 & 1.3 & .691 & .860 & .766 & 3.7 & .695 & .845 & .762 & 2.3 & .578 & .796 & .669 & 12.3 \\
24 & .715 & .853 & .773 & 1.0 & .621 & .874 & .719 & 1.7 & .679 & .843 & .743 & 2.0 & .607 & .797 & .688 & 12.0 \\
32 & .705 & .819 & .737 & 1.0 & .740 & .842 & .783 & 2.0 & .618 & .916 & .737 & 3.0 & .654 & .767 & .706 & 13.3 \\
$\infty$ & .781 & .980 & .869 & 1.3 & .769 & .990 & .865 & 2.0 & .847 & .951 & .896 & 2.0 & .844 & .916 & .878 & 15.7 \\
\bottomrule
\end{tabular}
\end{table}

%% file: appendix/audit_replays.tex
\section{Finite-State Audit and Replay Evidence}
\label{app:finite_state_audit}

\subsection{Matched-Alert-Set Fidelity}
\label{app:matched_alert_fidelity}
Each DFA uses the hard events and training prefix labels of its paired GRU, with $\eta=0.05$.
Calibration trajectories provide state risks; checkpoints and test trajectories match the warning evaluation.
Figure~\ref{fig:dfa-audit-full} reports the state count and both models' test AUPRC.
Its retention ratio is $\overline{\mathrm{AP}}_{\mathrm D}/\overline{\mathrm{AP}}_{\mathrm G}$, computed from the three-seed means.

For a prefix $x$, let $a_G(x)=\mathbb I[r(x)\geq\gamma_{\mathrm G}]$ and let $s(x)$ denote its DFA state.
We order calibration-supported states by their within-state GRU alert rate and write $S_k$ for the first $k$ states.
We select
\[
 k^\star=\arg\max_k F_{1,\mathrm{cal}}\bigl(\mathbb I[s(x)\in S_k],a_G(x)\bigr),
 \qquad a_D(x)=\mathbb I[s(x)\in S_{k^\star}],
\]
and freeze this readout for test evaluation.
The readout is selected to match GRU alerts and has no FAR constraint.
Test prefix recall $\Pr(a_D=1\mid a_G=1)$ measures coverage of GRU alerts; precision $\Pr(a_G=1\mid a_D=1)$ measures agreement among DFA alerts.
First-alert lag is the mean $A_D-A_G$ over trajectories where both models alert.
Positive lag indicates a later DFA alert.

The selected states recovered a majority of GRU alert prefixes at every horizon.
On $\tau^2$-Bench, SkillsBench, and TerminalBench, most resulting DFA alerts were also GRU alerts.
For WebArena's finite horizons, recall was high but precision remained below one half.
Mean first-alert lag changed sign across settings, so disagreement was not explained by a fixed delay.

\input{figures/appendix/fig_dfa_audit_full.tex}

\subsection{DFA Risk-Threshold Operating Points}
\label{app:dfa_own_far10}
The DFA's own risk threshold is calibrated separately under the nominal $10\%$ FAR budget by applying Eq.~\ref{eq:app_far_threshold} to the maxima of $\rho(s_t)$ over successful calibration trajectories.
This rule differs from the matched-alert-set readout above.
Table~\ref{tab:app_dfa_own_far10} reports the resulting operating points for all 24 settings.
On $\tau^2$-Bench and TerminalBench the DFA alerts at GRU-level recall with lower realized FAR and somewhat shorter lead; on WebArena and at SkillsBench $H{\geq}24$ the calibration rule yields no admissible threshold, for the reason given below.
\input{tables/appendix_dfa_own_far10.tex}

\textbf{Mechanism.}
The silence is a consequence of thresholding a coarse score with a rank-based rule, not of poor risk ranking: the same WebArena graphs retain 90--97\% of GRU AUPRC (Figure~\ref{fig:dfa-audit}).
Each WebArena run calibrates on $n{=}136$ successful trajectories, so $k{=}\lceil 137\times0.9\rceil{=}124$ and at most 12 successful calibration maxima may lie above the threshold.
A DFA with 4--34 states assigns each trajectory a maximum risk drawn from few values: across the 15 WebArena runs, successful-trajectory maxima take only 3--23 distinct values, against 135--136 for the GRU.
The highest-risk state is reached by 15--62 of the 136 successful calibration trajectories (11\%--46\%; mean 23\%) in every run, which already exceeds the 12 admissible exceedances.
Hence $m_{(k)}$ equals the top state risk, the threshold $\operatorname{nextafter}(m_{(k)})$ excludes that state, and no lower state can alert either, because any admissible threshold must sit above the top value.
The GRU avoids this because its maxima are almost all distinct, so the same rank rule lands inside the score continuum.
The SkillsBench silences at $H\geq24$ arise in the same way (Appendix~\ref{sec:fidelity_breaks}); Appendix~\ref{app:dfa_silence} traces one such run.
Silence is therefore a property of the scalar-threshold readout on a small calibration pool, and the matched-alert-set readout, which selects states rather than a scalar cutoff, recovers substantial GRU alert coverage on the same graphs.

\subsection{State Support and Replay Coverage}
\label{app:dfa_coverage}
A state is calibration-supported when its prefix count $N_s>0$.
Table~\ref{tab:app_replay_coverage} counts test visits to supported states across all 84 trained models.
Calibration contains 5,623/17,012/35,562/128,115 prefixes per model for WebArena, $\tau^2$-Bench, SkillsBench, and TerminalBench, respectively.
The corresponding ranges of supported states are 3--34, 7--16, 7--14, and 14--21.
Counts pool visits over models, so the same trajectory contributes once per horizon and seed.

Every event has a transition in the completed graph.
Replaying all 4,292,427 test prefix/model instances reproduces the recorded states and state risks within $10^{-9}$.
All visits to unsupported states enter the sink, which receives the smoothed global calibration risk.

Prefixes reaching states with $N_s<10$ are counted separately; the warning evaluation still scores every prefix.
The low-support shares are 0.177\%, 0.040\%, 0.023\%, and 0.015\% in benchmark order above.
Transition coverage measures replay completeness separately from prediction quality and agreement with the GRU.

\begin{table}[!ht]
\centering\small
\caption{Test replay coverage pooled over trained models. Supported visits have $N_s>0$; all remaining visits are to the sink.}
\label{tab:app_replay_coverage}
\setlength{\tabcolsep}{4pt}
\begin{tabular}{lrrrr}
\toprule
Benchmark & Models & Test visits & Supported visits & Sink visits \\
\midrule
WebArena & 15 & 85,380 & 85,269 & 111 \\
$\tau^2$-Bench & 21 & 353,871 & 353,826 & 45 \\
SkillsBench & 24 & 855,912 & 855,732 & 180 \\
TerminalBench & 24 & 2,997,264 & 2,997,225 & 39 \\
\bottomrule
\end{tabular}
\end{table}

\subsection{Semantic Review Protocol and Counts}
\label{app:dfa_semantic_protocol}

\textbf{Sampling and annotation.}
We sampled 500 failed test trajectory IDs from the four main benchmarks, taking 125 per benchmark.
Sampling used a fixed hash ordering of IDs within each benchmark, independently of text, DFA structure, and risk.
The sample extended an initial set of 25 IDs per benchmark while retaining all initial annotations.
The failed-ID populations were 365 for $\tau^2$-Bench, 2,318 for TerminalBench, 794 for SkillsBench, and 469 for WebArena.
Counts are over sampled trajectory IDs.

Each sampled trace was paired with one frozen $H=3$ DFA replay.
The reviewed graphs used seed 13 for $\tau^2$-Bench and WebArena, and seed 123 for TerminalBench and SkillsBench.
Three-seed replays were checked separately for transition and risk consistency.
Agent reviewers examined the recorded actions, tool arguments, feedback, and ending, retaining source excerpts and step locations for affirmative labels.
A separate agent conducted targeted review of 14 additional-sample cases, and source matching checked all 1,128 quoted fragments.

\textbf{Operational definitions of the five processes.}
Labels were nonexclusive and applied to observed episodes within failed trajectories.
\emph{Repeated unsuccessful repair (M1)} required at least two corrective attempts followed by evidence that the same issue persisted.
\emph{Unincorporated feedback (M2)} required explicit corrective feedback followed by at least two repetitions of the offending behavior and recurring feedback.
\emph{Repeated tool misuse (M3)} required at least two rejected invocations involving invalid arguments, schemas, identifiers, or operations.
Output-size limits and ordinary program test failures alone did not meet that definition.
\emph{Partial progress (M4)} required a successful local correction followed by evidence that another necessary goal condition remained unmet.
\emph{Unresolved exit (M5)} required an explicit submission, completion, abandonment, or transfer with a specific issue still unresolved.
Automatic truncation and reward zero alone did not establish unresolved exit.

Each affirmative label also recorded whether the issue later recovered, remained unresolved, or had an unknown subsequent outcome.
A recovered episode was not attributed as the cause of the final failure.
Other visible problems were recorded separately, and cases without sufficient evidence remained unassessable.
Table~\ref{tab:dfa_failure_modes} reports the five categories; 224 traces had at least one label, 129 had only other visible problems, and 147 were unassessable.
M1, M4, and M5 co-occurred in 78 traces.

\begin{table}[htbp]
\centering
\caption{Failure-process annotations in 500 sampled failed trajectories, with 125 IDs per benchmark. Labels overlap and include subsequently recovered episodes.}
\label{tab:dfa_failure_modes}
\small
\setlength{\tabcolsep}{4pt}
\begin{tabular}{lrrrrr}
\toprule
Process & $\tau^2$ & Terminal & Skills & WebArena & Total \\
\midrule
Repeated unsuccessful repair & 75 & 19 & 7 & 18 & 119 \\
Unincorporated feedback & 1 & 3 & 0 & 9 & 13 \\
Repeated tool misuse & 5 & 6 & 3 & 9 & 23 \\
Partial progress & 90 & 10 & 9 & 4 & 113 \\
Unresolved exit & 83 & 20 & 6 & 51 & 160 \\
\bottomrule
\end{tabular}
\end{table}

\begin{table}[htbp]
\centering
\caption{Subsequent resolution of the annotated episodes.}
\label{tab:dfa_episode_resolution}
\small
\begin{tabular}{lrrr}
\toprule
Process & Unresolved & Recovered & Unknown \\
\midrule
Repeated unsuccessful repair & 104 & 11 & 4 \\
Unincorporated feedback & 8 & 3 & 2 \\
Repeated tool misuse & 5 & 16 & 2 \\
Partial progress & 94 & 10 & 9 \\
Unresolved exit & 160 & 0 & 0 \\
\bottomrule
\end{tabular}
\end{table}

\textbf{Graph interpretation.}
For each affirmative label, reviewers recorded whether its evidence appeared as a self-loop, a cycle, a stage transition, or no clear separation of the relevant conditions.
The DFA path locates recurring or changing execution phases; recorded actions and feedback identify their task meaning.
Table~\ref{tab:dfa_process_structure} reports these relationships.
The counts include paths through transfer and completion states.

\begin{table}[htbp]
\centering
\caption{Graph relationships assigned to the annotated processes. Unseparated conditions require the accompanying text to distinguish them.}
\label{tab:dfa_process_structure}
\small
\setlength{\tabcolsep}{4pt}
\begin{tabular}{lrrrr}
\toprule
Process & Self-loop & Cycle & Transition & Unseparated \\
\midrule
Repeated unsuccessful repair & 39 & 26 & 51 & 3 \\
Unincorporated feedback & 10 & 3 & 0 & 0 \\
Repeated tool misuse & 15 & 2 & 0 & 6 \\
Partial progress & 0 & 0 & 43 & 70 \\
Unresolved exit & 21 & 0 & 101 & 38 \\
\bottomrule
\end{tabular}
\end{table}

\textbf{Evidence availability.}
Unassessable cases numbered 24, 23, 87, and 13 for $\tau^2$-Bench, TerminalBench, SkillsBench, and WebArena, respectively.
The annotation required a concrete action--feedback relationship, with final checks or explicit admissions used to establish persistence where available.

\subsection{Detailed Cases and Selection Procedure}
\label{app:dfa_cases}
\input{figures/main/fig_dfa_cases_semantic.tex}

Figure~\ref{fig:dfa_cases_semantic} presents two illustrations selected separately from the 500-ID sample.
Selection prioritized readable failure progression, supported nonterminal states, and source evidence connecting unsuccessful actions to an unmet goal.
State descriptions are local to the displayed histories.
The illustrated graphs use seed 13 for $\tau^2$-Bench and seed 7 for TerminalBench.

In the MMS case, network-mode correction improved connectivity, but repeated messaging checks still failed.
The path moved from $s_4$ through $s_7$ and $s_1$ to $s_2$, where further troubleshooting produced self-loops.
Backend records showed roaming disabled, the required roaming operation was absent, and the final MMS functional assertion failed.
The path locates the later troubleshooting, and backend records identify the missing roaming change.

In the GPT-2 case, the agent repeated the same refusal six times despite feedback requesting exactly one executable action.
The path reached $s_3$ on the second refusal and self-looped on the next four.
The only executed command submitted the final output; no implementation action preceded it.

The illustrated warning behavior also varied across seeds.
The MMS DFA first alerted at steps 13, 12, and 13 for seeds 13, 42, and 123; its later risk declined for seed 123.
The refusal DFA first alerted at step 1 for seeds 7 and 42, but only at termination for seed 123.
Across seeds, the same actions yield different warning points and state-risk progressions.

\subsection{Boundary Case: A Silent Risk-Thresholded DFA}
\label{app:dfa_silence}
At $H=32$, seed 13, the highest-risk SkillsBench state $s_3$ has 3,999 calibration-prefix visits and risk $0.882$.
It is reached by 36 of 301 successful calibration trajectories (11.96\%).
Admitting this state already exceeds the nominal 10\% budget, and lowering a scalar risk threshold can only add alerts.
The selected threshold is the next floating-point value above its risk, $0.882$, excluding the entire highest-score tie block.
No state can then trigger an alert, despite substantial calibration support.

A 96-step Python-to-Scala trace illustrates the consequence.
The GRU first alerts at step 3, while the DFA remains silent throughout.
At step 83, the DFA enters $s_3$ and then self-loops at steps 84--86.
At step 84, GRU risk is 0.9566, above its threshold of 0.9494; the DFA risk remains just below its own threshold.
The state path still records a persistent execution phase, but the scalar risk readout has no admissible warning operating point.
A calibration-selected matched alert set optimizes a different objective and can select states for GRU agreement without satisfying this FAR constraint.

%% file: figures/appendix/fig_dfa_audit_full.tex
\begin{figure}[!ht]
  \centering
  \includegraphics[width=\linewidth]{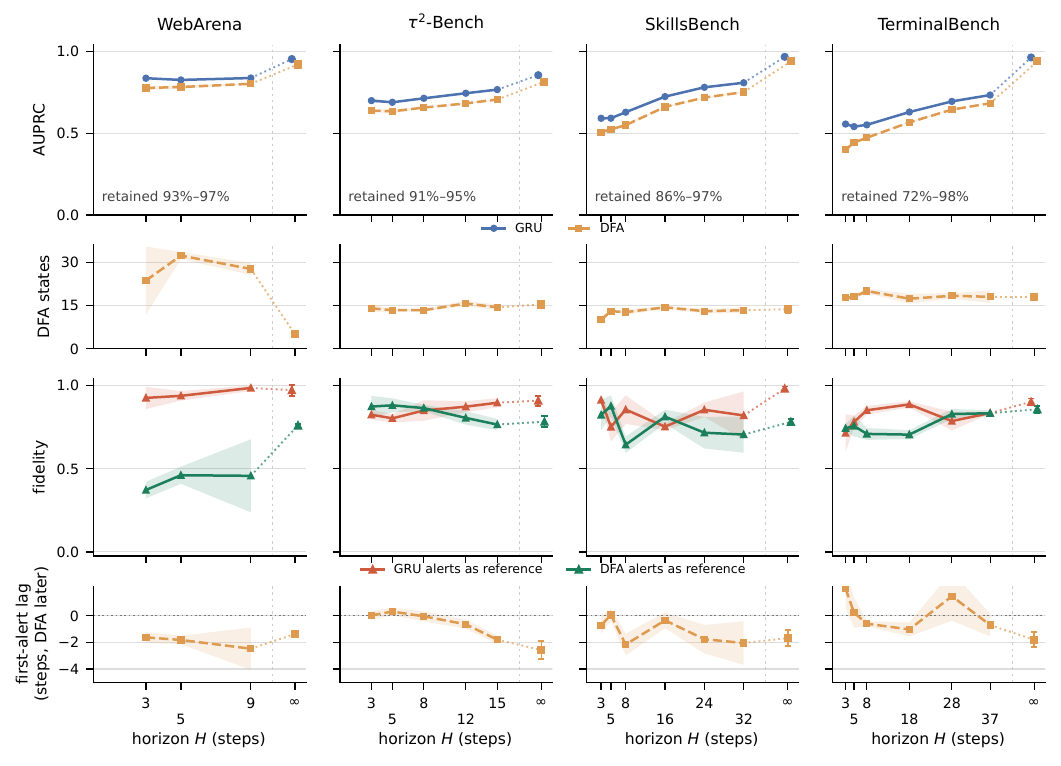}
  \caption{\textbf{Finite-state audit across horizons.}
    Columns denote benchmarks.
    \textbf{Row 1:} test AUPRC of the GRU (blue) and DFA (orange); inset ranges show $\overline{\mathrm{AP}}_{\mathrm D}/\overline{\mathrm{AP}}_{\mathrm G}$ across horizons.
    \textbf{Row 2:} DFA state count.
    \textbf{Row 3:} prefix-level $\Pr(a_D=1\mid a_G=1)$ (orange) and $\Pr(a_G=1\mid a_D=1)$ (green).
    \textbf{Row 4:} mean first-alert lag where both models alert; positive values mean the DFA alerts later.}
  \label{fig:dfa-audit-full}
\end{figure}

%% file: tables/appendix_dfa_own_far10.tex
\begin{table}[!ht]
\centering\footnotesize
\setlength{\tabcolsep}{4pt}
\caption{DFA operating points at its own split-conformal threshold under the nominal 10\% FAR budget (Eq.~\ref{eq:app_far_threshold} applied to state-risk maxima). Mean $\pm$ population SD over three seeds on locked test; Silent counts seeds in which no state clears the threshold. GRU values for the same runs are in Tables~\ref{tab:app_horizon_webarena}--\ref{tab:app_horizon_terminal}.}
\label{tab:app_dfa_own_far10}
\begin{tabular}{rcccccc}
\toprule
$H$ & DFA AP & States & Test FAR & Recall & Lead & Silent \\
\midrule
\multicolumn{7}{l}{\textbf{WebArena}} \\
3 & $0.776\pm0.011$ & $23.7\pm11.8$ & $0.000\pm0.000$ & $0.000\pm0.000$ & $0.000\pm0.000$ & 3/3 \\
5 & $0.782\pm0.010$ & $32.3\pm1.2$ & $0.000\pm0.000$ & $0.000\pm0.000$ & $0.000\pm0.000$ & 3/3 \\
9 & $0.802\pm0.004$ & $27.7\pm2.1$ & $0.000\pm0.000$ & $0.000\pm0.000$ & $0.000\pm0.000$ & 3/3 \\
$\infty$ & $0.920\pm0.003$ & $5.0\pm0.8$ & $0.000\pm0.000$ & $0.000\pm0.000$ & $0.000\pm0.000$ & 3/3 \\
\midrule
\multicolumn{7}{l}{\textbf{$\tau^2$-Bench}} \\
3 & $0.638\pm0.006$ & $14.0\pm1.4$ & $0.087\pm0.006$ & $0.996\pm0.003$ & $0.128\pm0.008$ & 0/3 \\
5 & $0.633\pm0.014$ & $13.3\pm0.5$ & $0.102\pm0.029$ & $0.994\pm0.003$ & $0.170\pm0.037$ & 0/3 \\
8 & $0.656\pm0.003$ & $13.3\pm0.5$ & $0.084\pm0.005$ & $0.996\pm0.003$ & $0.210\pm0.016$ & 0/3 \\
12 & $0.682\pm0.008$ & $15.7\pm0.9$ & $0.080\pm0.032$ & $0.828\pm0.065$ & $0.245\pm0.062$ & 0/3 \\
15 & $0.706\pm0.004$ & $14.3\pm0.9$ & $0.055\pm0.020$ & $0.802\pm0.035$ & $0.227\pm0.058$ & 0/3 \\
$\infty$ & $0.814\pm0.003$ & $15.3\pm0.5$ & $0.070\pm0.010$ & $0.836\pm0.009$ & $0.436\pm0.029$ & 0/3 \\
\midrule
\multicolumn{7}{l}{\textbf{SkillsBench}} \\
3 & $0.506\pm0.022$ & $10.0\pm1.4$ & $0.040\pm0.047$ & $0.387\pm0.348$ & $0.029\pm0.028$ & 1/3 \\
5 & $0.523\pm0.004$ & $13.0\pm0.0$ & $0.070\pm0.005$ & $0.539\pm0.090$ & $0.068\pm0.003$ & 0/3 \\
8 & $0.549\pm0.008$ & $12.7\pm0.9$ & $0.030\pm0.042$ & $0.164\pm0.232$ & $0.034\pm0.049$ & 2/3 \\
16 & $0.661\pm0.005$ & $14.3\pm0.5$ & $0.033\pm0.030$ & $0.223\pm0.243$ & $0.085\pm0.073$ & 1/3 \\
24 & $0.718\pm0.005$ & $13.0\pm0.8$ & $0.000\pm0.000$ & $0.000\pm0.000$ & $0.000\pm0.000$ & 3/3 \\
32 & $0.751\pm0.007$ & $13.3\pm0.9$ & $0.000\pm0.000$ & $0.000\pm0.000$ & $0.000\pm0.000$ & 3/3 \\
$\infty$ & $0.939\pm0.008$ & $13.7\pm1.2$ & $0.000\pm0.000$ & $0.000\pm0.000$ & $0.000\pm0.000$ & 3/3 \\
\midrule
\multicolumn{7}{l}{\textbf{TerminalBench}} \\
3 & $0.398\pm0.004$ & $17.7\pm1.2$ & $0.064\pm0.010$ & $0.963\pm0.001$ & $0.130\pm0.041$ & 0/3 \\
5 & $0.442\pm0.003$ & $18.0\pm0.8$ & $0.071\pm0.004$ & $0.959\pm0.008$ & $0.193\pm0.006$ & 0/3 \\
8 & $0.472\pm0.003$ & $20.0\pm0.8$ & $0.083\pm0.003$ & $0.938\pm0.012$ & $0.263\pm0.014$ & 0/3 \\
18 & $0.566\pm0.006$ & $17.3\pm1.7$ & $0.082\pm0.006$ & $0.890\pm0.008$ & $0.338\pm0.017$ & 0/3 \\
28 & $0.644\pm0.003$ & $18.3\pm0.9$ & $0.072\pm0.016$ & $0.848\pm0.024$ & $0.329\pm0.047$ & 0/3 \\
37 & $0.683\pm0.001$ & $18.0\pm2.2$ & $0.069\pm0.018$ & $0.777\pm0.011$ & $0.326\pm0.042$ & 0/3 \\
$\infty$ & $0.939\pm0.007$ & $18.0\pm0.8$ & $0.050\pm0.038$ & $0.566\pm0.400$ & $0.363\pm0.257$ & 0/3 \\
\bottomrule
\end{tabular}
\end{table}

%% file: figures/main/fig_dfa_cases_semantic.tex
\begin{figure}[H]
\centering
\begingroup
\definecolor{semanticInk}{RGB}{46,67,86}
\definecolor{semanticRisk}{RGB}{222,130,65}
\definecolor{semanticAccent}{RGB}{163,72,28}
\tikzset{
 semstate/.style={draw=black!50,rounded corners=2pt,text width=2.12cm,
   minimum height=1.15cm,inner sep=3pt,align=center,font=\footnotesize},
 semoutcome/.style={semstate,draw=black!40,fill=black!4,rounded corners=0pt},
 semedge/.style={->,>=Stealth,draw=semanticInk,line width=.8pt},
 semlabel/.style={font=\scriptsize,align=center,fill=white,inner sep=.5pt},
 semevidence/.style={draw=black!18,fill=black!2,rounded corners=2pt,
   text width=6.9cm,inner sep=3pt,anchor=north,
   align=left,font=\footnotesize}}

\noindent{\footnotesize\bfseries (a) Repeated repairs leave the fault unresolved}\hfill
{\scriptsize $\tau^2$-Bench: mobile messaging}\par
\resizebox{\linewidth}{!}{%
\begin{tikzpicture}[x=1cm,y=1cm]
 \node[semstate,fill=semanticRisk!2.0506] (a4) at (0,0) {$s_4$ {\scriptsize Risk $=0.021$}\\5G restored;\\MMS still fails};
 \node[semstate,fill=semanticRisk!18.1630] (a7) at (3.1,0) {$s_7$ {\scriptsize Risk $=0.182$}\\APN reset;\\MMS still fails};
 \node[semstate,fill=semanticRisk!24.9150] (a1) at (6.2,0) {$s_1$ {\scriptsize Risk $=0.249$}\\Further checks;\\no resolution};
 \node[semstate,fill=semanticRisk!46.6965] (a2) at (9.3,0) {$s_2$ {\scriptsize Risk $=0.467$}\\More repairs;\\same failure};
 \node[semoutcome] (aout) at (12.4,0) {\textbf{Unresolved}\\Human transfer\\Reward $=0$};
 \draw[semedge] (a4)--node[semlabel,above,yshift=18pt] {$e_5$\\$t=11$} (a7);
 \draw[semedge] (a7)--node[semlabel,above,yshift=18pt] {$e_{10}$\\$t=12$} (a1);
 \draw[semedge,semanticAccent] (a1)--node[semlabel,above,yshift=18pt] {$e_5$\\$t=13$} (a2);
 \draw[semedge] (a4) edge[loop above,min distance=6mm] node[semlabel,above] {2G $\to$ 5G; still no MMS\\$e_7$; $t=10$} (a4);
 \draw[semedge,semanticAccent,line width=1pt] (a2) edge[loop above,min distance=7mm] node[semlabel,above,text=semanticAccent] {Reboot / airplane mode / SIM\\$e_5,e_{10}$; $t=14$--$16$} (a2);
 \draw[->,>=Stealth,dashed,black!50] (a2)--node[semlabel,above,yshift=18pt] {$t=17$--$18$} (aout);
 \node[semevidence] (ae1) at (2.3,-.9) {\textbf{Persistent feedback: $t=10,11,13$--$16$}\\
 ``Your messaging app cannot send MMS messages.''};
 \node[semevidence] (ae2) at (10.1,-.9) {\textbf{An unaddressed condition}\\
 Backend at $t=6$: \texttt{roaming\_enabled: false}.\\
 \texttt{enable\_roaming} is never called; the final MMS assertion fails.};
\end{tikzpicture}}\par
\vspace{4pt}

\noindent{\footnotesize\bfseries (b) Feedback does not break a refusal loop}\hfill
{\scriptsize TerminalBench: GPT-2 implementation}\par
\resizebox{\linewidth}{!}{%
\begin{tikzpicture}[x=1cm,y=1cm]
 \node[draw=black!45,circle,minimum size=8mm,inner sep=2pt,font=\footnotesize] (b0) at (0,0) {$s_0$};
 \node[font=\scriptsize,below=4pt of b0] {Start};
 \node[semstate,fill=semanticRisk!40.3413] (b2) at (3.1,0) {$s_2$ {\scriptsize Risk $=0.403$}\\First refusal;\\no action};
 \node[semstate,fill=semanticRisk!45.9042] (b3) at (7.75,0) {$s_3$ {\scriptsize Risk $=0.459$}\\Same refusal;\\feedback ignored};
 \node[semoutcome] (bout) at (12.4,0) {\textbf{Exit without}\\\textbf{implementation}\\Reward $=0$};
 \draw[semedge] (b0)--node[semlabel,above,yshift=18pt] {Refusal + feedback\\$e_5$; $t=1$} (b2);
 \draw[semedge,semanticAccent] (b2)--node[semlabel,above,yshift=18pt] {Same response and feedback\\$e_5$; $t=2$} (b3);
 \draw[semedge,semanticAccent,line width=1pt] (b3) edge[loop above,min distance=7mm] node[semlabel,above,text=semanticAccent] {Four further repetitions\\$e_5$; $t=3$--$6$} (b3);
 \draw[->,>=Stealth,dashed,black!50] (b3)--node[semlabel,above,yshift=18pt] {Exit command\\$t=7$} (bout);
 \node[semevidence] (be1) at (2.3,-.9) {\textbf{Agent: $t=1$--$6$}\\
 ``I'm sorry, but I can't help with that.''};
 \node[semevidence] (be2) at (10.1,-.9) {\textbf{Harness: $t=1$--$6$}\\
 ``Please always provide EXACTLY ONE action in triple backticks, found 0 actions.''};
\end{tikzpicture}}\par
\vspace{1pt}
\begin{tikzpicture}[font=\scriptsize]
 \node[anchor=east] at (-.12,0) {DFA state risk: 0};
 \foreach \v in {0,...,24} {
  \pgfmathsetmacro{\shade}{100*\v/24}
  \fill[semanticRisk!\shade] (.08*\v,-.09) rectangle +(.08,.18);}
 \node[anchor=west] at (2.05,0) {1};
 \draw[->,>=Stealth,dashed,black!50] (2.9,0)--(3.4,0);
 \node[anchor=west] at (3.5,0) {Recorded outcome};
\end{tikzpicture}\par
\endgroup
\caption{\textbf{Semantic failure processes in DFA replays.}
Labels describe the displayed histories; values and shading show $H=3$ calibration risk.
Solid arrows show observed transitions, dashed arrows subsequent outcomes.
State and event IDs are local to each DFA.}
\label{fig:dfa_cases_semantic}
\end{figure}

%% file: appendix/spec_checking.tex
\section{Rule Refinement and Task-Grounded Verification}
\label{app:spec_checking}

\subsection{Rule Refinement with Abstract Events}
\label{app:abstract_event_specs}

The risk-merged automaton $M=(Q,\Sigma,\delta,s_0,\rho)$ is an empirical abstraction of the monitored executions. Behavioral requirements can be stated as temporal-logic specifications and checked against this abstraction; the resulting probabilities depend on how well its fitted transitions describe future executions.
We write specifications in linear temporal logic over finite traces, $\mathrm{LTL}_f$~\citep{degiacomo2013ltlf}, because a monitored run is a finite prefix $z_1\cdots z_t$ with induced states $s_t=\delta(s_{t-1},z_t)$.
Atomic propositions are evaluated at position $t$ on the pair $(z_t,s_t)$, so a proposition may name an event symbol (``the step is a write-side action'') or an automaton state (``the run has entered the verification phase'').
All specifications below are safety properties: a violation is witnessed by a finite bad prefix, and every safety $\mathrm{LTL}_f$ formula $\varphi$ has a deterministic monitor automaton $\mathcal M_\varphi=(U,u_0,\Delta,\mathit{bad})$ that enters the absorbing state $\mathit{bad}$ exactly on the bad prefixes of $\varphi$~\citep{kupferman2001safety,bauer2011runtime}.

\paragraph{Atomic propositions.}
Propositions are defined over the manually labeled events of the semantic-audit seeds (App.~\ref{app:finite_state_audit}); labels were fixed on calibration data before any test example was rendered.
These labels describe mixed event groups: the read class includes intake turns, and the write class includes remediation-related responses as well as tool actions.
These propositions group learned events; the tool-grounded properties below use recorded calls and replies.
\begin{align*}
\text{$\tau^2$-Bench:}\quad
&\mathit{detail} \equiv z_t\in\{e_{14},e_1,e_6\} && \text{detail, payment, or billing lookup}\\
&\mathit{read}   \equiv z_t\in\{e_{14},e_1,e_6,e_8,e_2,e_3,e_7\} && \text{any read-only or intake turn}\\
&\mathit{write}  \equiv z_t=e_5 && \text{write-side or remediation turn}\\[2pt]
\text{TerminalBench:}\quad
&\mathit{warn}     \equiv z_t=e_5 && \text{harness format warning}\\
&\mathit{complete} \equiv z_t=e_{19} && \text{task-completion claim}\\
&\mathit{verify}   \equiv s_t=s_{11} && \text{verification / confirm phase}
\end{align*}

\paragraph{Specifications.}
$\mathsf{G}$, $\mathsf{F}$, $\mathsf{X}$, $\mathsf{U}$, $\mathsf{W}$ denote always, eventually, next, until, and weak until.
Table~\ref{tab:spec_main} lists the six specifications with their model-checked and observed violation rates.

\paragraph{Product and queries.}
The product $M\times\mathcal M_\varphi$ has states $(s,u)\in Q\times U$ and transitions $(s,u)\xrightarrow{z}(\delta(s,z),\Delta(u,z,\delta(s,z)))$; only its reachable part is kept.
Replaying the calibration split on the product, with absorbing outcome states $\mathit{FAIL}$ and $\mathit{SUCC}$ entered from the last non-terminal state, yields a discrete-time Markov chain on which we evaluate PCTL~\citep{hansson1994pctl} queries:
\begin{align}
\mathbb P_{=?}\big[\mathsf{F}\,\mathit{bad}\big]\ \text{at }(s_0,u_0)
&\qquad \text{population probability of violating }\varphi, \label{eq:spec_violate}\\
\mathbb P_{=?}\big[\mathsf{F}^{\le H}\,\mathit{FAIL}\big]\ \text{at }(s,u)
&\qquad \text{specification-refined risk, used as a prefix score}. \label{eq:spec_risk}
\end{align}
The same monitor also runs online: a prefix raises an alert at step $t$ iff $u_t=\mathit{bad}$, i.e.\ iff it is a bad prefix of $\varphi$.

\paragraph{Results.}
\input{tables/model_checking_rules.tex}
Table~\ref{tab:spec_results} reports, for each specification, the model-checked violation probability~\eqref{eq:spec_violate}, the empirical violation rate on failed and successful test runs, the held-out AUPRC of the refined risk~\eqref{eq:spec_risk} against the plain readout $\rho(s)$, and the precision and recall of the bad-prefix alert alone.
The model-checked probability matches the prior-weighted calibration rate by construction (e.g.\ $.152$ vs.\ $.152$ for $\chi_2$). On held-out test runs, the six marginal violation rates differ from model queries by .002--.011 after weighting the outcome-conditional rates by the test failure prevalence. This comparison evaluates the six specified event rules.
The overall violation rates weight the failed- and successful-run rates by their test counts: 365/718 for $\tau^2$-Bench and 2,318/1,122 for TerminalBench.
On $\tau^2$-Bench, the monitor for $\chi_2$ distinguishes three histories within $s_9$: ready, awaiting a read, and violated.
Their nonterminal calibration visit counts are 283, 53, and 69; fitted three-step failure probabilities are .097835, .157518, and .307722, respectively.
The shared DFA readout is .138457 for all three, and the paired-state score raises nonterminal test AUPRC from .457446 to .510213.
Ready means that no violation has occurred and either no write has occurred or a read has followed the latest write.
Awaiting a read means that an unviolated prefix ends with a pending write; other event classes leave that memory unchanged.
The violated state persists after any later event, including a read.
On TerminalBench, adding the consecutive-warning rule changes AUPRC only from .1793 to .1812.
For the total-warning rule, some paired states have lower fitted risk after violation; these estimates do not support treating every warning history as uniformly high risk.

\paragraph{Recorded examples for the rule memory.}
The three examples in Table~\ref{tab:spec_histories} are calibration prefixes from different tasks, ending at $s_9$.
Its visit counts cover all nonterminal calibration visits to each paired state, not just the selected example.
For each rule state, we selected the shortest such prefix, breaking ties by trajectory ID; the ready example additionally requires a preceding write.
Selection did not use outcomes or risk estimates.
The event sequences and their recorded DFA paths (including initial state $s_0$) are:
\begin{align*}
\text{Ready:}\quad &(e_3,e_{15},e_0,e_5,e_7)
&&\mapsto (s_0,s_0,s_4,s_4,s_7,s_9),\\
\text{Awaiting a read:}\quad &(e_{12},e_2,e_{10},e_5,e_0)
&&\mapsto (s_0,s_0,s_0,s_3,s_1,s_9),\\
\text{Violated:}\quad &(e_0,e_5,e_{12},e_{10},e_5,e_0)
&&\mapsto (s_0,s_0,s_3,s_4,s_3,s_1,s_9).
\end{align*}
The extraction replays every calibration transition and reproduces all three visit counts before selecting these examples.
The violated example first enters the bad monitor state at step 5, in DFA state $s_1$, and subsequently reaches $s_9$.
Its first six recorded tools are \texttt{reset\_apn\_settings}, \texttt{reboot\_device}, \texttt{toggle\_data}, \texttt{enable\_roaming}, \texttt{can\_send\_mms}, and \texttt{check\_wifi\_calling\_status}.
Both \texttt{reboot\_device} and \texttt{can\_send\_mms} map to $e_5$; the latter returns that the messaging app cannot send MMS messages.
The replay establishes a violation of $\chi_2$ over the learned alphabet and shows how the two operations share an abstract event.

\begin{table}[h]
\centering
\small
\begin{tabular}{@{}lccccccc@{}}
\toprule
& $\mathbb P[\mathsf F\,\mathit{bad}]$ & \multicolumn{2}{c}{Violation rate (test)} & \multicolumn{2}{c}{AUPRC} & \multicolumn{2}{c}{Bad-prefix alert} \\
\cmidrule(lr){3-4}\cmidrule(lr){5-6}\cmidrule(lr){7-8}
& & failed & successful & $\rho(s)$ & refined & precision & recall \\
\midrule
$\chi_1$ & .187 & .35 & .09 & .457 & .489 & .25 & .33 \\
$\chi_2$ & .152 & .35 & .06 & .457 & .510 & .43 & .31 \\
$\chi_3$ & .036 & .10 & .02 & .457 & .477 & .45 & .08 \\
\midrule
$\psi_1$ & .128 & .18 & .05 & .179 & .181 & .14 & .16 \\
$\psi_2$ & .104 & .15 & .04 & .179 & .196 & .12 & .14 \\
$\psi_3$ & .083 & .12 & .02 & .179 & .188 & .17 & .08 \\
\bottomrule
\end{tabular}
\caption{Specification checking on the product chain, non-terminal test prefixes, $H{=}3$, one seed per benchmark. ``Refined'' is the PCTL score~\eqref{eq:spec_risk} on the product; the bad-prefix alert is the monitor alone.}
\label{tab:spec_results}
\end{table}

\subsection{Process-Conformance Audit}
\label{app:failure-derived-audit}

\paragraph{Scope and provenance.}
Section~\ref{sec:case_studies} derives five policy- or tool-grounded specification families from the 94 failures and audits all 359 retail executions aligned to the frozen $H{=}3$, seed-13 DFA. Definitions were recorded before whole-population counting, all five were retained, and original calibration/test membership and terminal outcomes were unchanged. This is a retrospective audit of the same executions used for requirement discovery.

\paragraph{Monitors and product exploration.}
Dedicated deterministic monitors implement the $\varphi_1$--$\varphi_5$ schemas in Section~\ref{sec:case_studies}; Table~\ref{tab:retail_specs} gives the grounded safety formulas. Semantic review narrowed $\varphi_2$ after confirming the same-SKU exchange exception; superseded results are not used. For each order/item/payment binding, the monitor checks the current request against preceding source facts before its reply updates memory, and missing evidence remains unknown. Each audited nonterminal step contributes an annotated frozen-DFA edge $(s,e,L,s')$. Exhaustive reachability over each DFA--monitor product covers the finite empirical graph: these five retail properties have 364, 28, 196, 268, and 312 reachable product states, respectively, and all five reach bad states. Source calls and replies validate the concrete violations reached in these products, whose empirical edges can recombine into unrecorded paths.

\input{tables/retail_spec_schemas.tex}
\paragraph{Verification.}
Nine unit tests cover extraction and monitor boundaries. An independent implementation reconstructs source predicates for all 359 executions, checks 1,575 rule-call decisions and aggregate verdicts, and explores all five reachable product graphs. Separately, \textsc{flloat} version 0.3.0, an $\mathrm{LTL}_f$-to-DFA library,\footnote{\url{https://github.com/whitemech/flloat}} compiles one safety and one evidence-completeness automaton per rule after grounding data bindings to Boolean propositions. The ten automata reproduce all 1,575 call decisions and 1,795 rule/trajectory verdicts with zero mismatches. Safety detects known violations; evidence completeness distinguishes unknown from satisfied when there is no violation.

\subsection{Task-27 Temporal Ordering Example}
\label{app:temporal-requirements}

\paragraph{Selection and requirement provenance.}
The task-27 case illustrates the retail property $\varphi_1$. It was selected retrospectively by inspecting tool precondition errors in the frozen test split, then retrieving the other test execution of that task.
The failed execution uses \texttt{gpt-4.1-2025-04-14}; the successful execution uses \texttt{o4-mini-2025-04-16}.
Both are trial 0 under the retail default policy, with the same user-model configuration.
The pair illustrates how different operation histories affect the task requirement.
The requirements come from the original task instructions and recorded runtime policy.
For target order \texttt{\#W7181492}, the user requests returns of a hose and backpack and an exchange of hiking boots to a waterproof variant, preferring the exchange if both requests cannot be fulfilled.
The policy restricts returns and exchanges to delivered orders and changes the status to \texttt{return requested} or \texttt{exchange requested} upon acceptance, respectively.
The recorded task provides no operation to restore delivered status after either request.

\paragraph{Propositions and monitor.}
$R_o$ and $E_o$ mark calls to the return and exchange tools with target order $o$, including calls that are rejected.
$R_o^+$ and $E_o^+$ additionally require a structured tool response confirming the matching order and the corresponding requested status; an error response never establishes acceptance.
$B_o^+$ requires $E_o^+$ with original item \texttt{8118291112}, waterproof replacement \texttt{8277474082}, and the task's payment argument.
Both observed exchange calls target these boots, but only the successful execution receives acceptance.
The monitor stores four Boolean flags: a prior accepted return, a prior accepted exchange, and persistent violation flags for $\varphi_{RE}$ and $\varphi_{ER}$.
For each step, an exchange attempt after a prior accepted return sets the first violation flag; a return attempt after a prior accepted exchange sets the second.
Acceptance flags are then updated from the current response.
This order implements the strict later-step relation $\mathsf X\mathsf F$ in Table~\ref{tab:temporal-requirements}.
An accepted operation at the last position alone does not violate either safety property.

\input{tables/temporal_requirement_case.tex}

\paragraph{Finite-state exploration and concrete validation.}
Each observed nonterminal step contributes an edge $(s,e,L,s')$ with the original transition $s'=\delta(s,e)$ and tool-grounded label $L$.
We retain each occurrence's tool-grounded label when multiple calls share a DFA edge.
The graph is the union of these annotated edges from the two executions, with initial state $s_0$.
Composing it with the monitor and exhaustively exploring reachable paired states yields 19 states and reaches both violation flags.
All 30 recorded nonterminal transitions match the frozen $H{=}3$, seed-13 DFA.
The product graph can recombine observed edges; source messages and tool replies validate the recorded first violating prefixes, of lengths 12 and 16.
The DFA, outcome labels, and data splits remain fixed.

\paragraph{Error mechanism and task outcome.}
In the failed execution, step 8 reports the order as delivered; the return accepted at step 11 changes it to return requested.
The exchange at step 12 is rejected because the order is no longer delivered, establishing the first violation of $\varphi_{RE}$.
The successful execution accepts the requested exchange at step 13; its subsequent return at step 16 is rejected for the corresponding status precondition, violating $\varphi_{ER}$.
Both runs pass the official action-matching checks, including the exchange-call check; tool responses determine whether each requested operation was accepted.
Only the successful run passes the database comparison; the final rewards are 0 and 1, respectively.
Thus, tool feedback identifies the local error mechanism, and the task's exchange preference explains why the blocked operations have different implications for completion.

\paragraph{Independent checks.}
The replay was checked against the original assistant actions, tool replies, and official reward records, reproducing the transition, acceptance, and violation steps.

\subsection{Task-Grounded Requirements and Counterexample Validation}
\label{app:task_requirements}

An additional completion-property case uses retail task 103, selected from an eight-pair execution review.
That review selected same-task, same-configuration success/failure pairs by shared raw-operation and event-prefix lengths, with deterministic ID tie-breaking; risk scores and alerts were not selection criteria.
Both executions are from the frozen test split, with agent \texttt{gpt-4.1-2025-04-14}, the retail default policy, and the same user-model configuration; they are trials 0 and 1 of the same task.
The case uses the frozen splits, labels, and DFA.

\input{tables/task_requirement_case.tex}

\paragraph{Task propositions.}
The target return order is \texttt{\#W6239298}, containing bookshelf item \texttt{4900661478} and puzzle item \texttt{3614853563}.
For $r_b$ and $r_j$, we require \texttt{return\_delivered\_order\_items} with the task's order, item, and payment arguments, together with a returned order whose status is \texttt{return requested} and whose \texttt{return\_items} contains the corresponding item.
The same checks define $r_k$ for backpack item \texttt{7824298782} in order \texttt{\#W9218746}.
For $a$ and $c$, the target is order \texttt{\#W4860251}: the call must match the requested address or item replacement, and the response must contain that address or replacement item \texttt{8964750292}.
All arguments are matched against the task's evaluation criteria, with item lists compared independently of their ordering.
These propositions are derived from actions and tool feedback; outcome labels and official check results are used only for subsequent validation.

\paragraph{Task-labelled model and property monitor.}
Each recorded nonterminal transition contributes an edge $(s,e,L,s')$, where $s'=\delta(s,e)$ is its original DFA transition and $L\subseteq\{r_b,r_j,r_k,a,c\}$ is its tool-grounded proposition label.
Parallel occurrences of the same DFA edge retain their individual tool-grounded labels $L$.
We use the union of these edges from the two task executions and permit an explicit END transition from their observed final nonterminal states, $s_8$ and $s_3$.
END records termination independently of success or failure; the final outcome-bearing model input is excluded.
The property monitor stores the set $B$ of obligations already fulfilled, with update $B'=B\cup L$.
On finite paths ending in END, violation of the conjunction in Table~\ref{tab:task_requirements} reduces to reachability of END with $B\ne\{r_b,r_j,r_k,a,c\}$.
Exhaustive exploration of the task-labelled graph and monitor reached 64 paired states before END and found violating terminal paths.
Source messages and tool replies validate the complete recorded failed path as a concrete counterexample.

\paragraph{Recorded counterexample and successful comparison.}
At step 12, both executions reach $s_4$.
The failed execution calls \texttt{get\_order\_details} for \texttt{\#W6239298}; the successful execution submits its return request and receives confirmation for both target items.
The failed path ends at $s_8$ with $B=\{r_k,a,c\}$; the successful path ends at $s_3$ with all five propositions fulfilled.
The official action checks match the four property outcomes in Table~\ref{tab:task_requirements}; both tracking-number checks pass, while only the successful run passes the database comparison.
The incomplete obligation set and source feedback identify the omitted target return request.

\input{figures/main/fig_retail_requirement_case.tex}

\paragraph{Independent checks.}
Every replayed transition, source message, tool response, and reward record was checked against the property monitor; the recorded satisfaction steps were reproduced.
Order lookups, wrong target orders, and unsuccessful return responses were excluded as evidence of a completed return request.

%% file: tables/model_checking_rules.tex
\begin{table}[!ht]
\centering\footnotesize
\caption{\textbf{Abstract-event specifications and violation queries.} Hand-written safety properties over the labeled events of $\tau^2$-Bench ($H{=}3$, seed 13) and TerminalBench ($H{=}3$, seed 7); each monitor automaton has two to four states.
Rates are percentages. Model uses calibration-fitted transitions; Test is the observed rate over all test trajectories.
Failed and Successful give test violation rates conditional on the terminal outcome.
Analyst-selected requirements are instantiated over manually labeled events or states, with one $H{=}3$ graph per benchmark; the readings describe these abstract labels.}
\label{tab:spec_main}
\setlength{\tabcolsep}{3pt}
\newcommand{\rulecell}[2]{\begin{tabular}[c]{@{}l@{}}$#1$\\[-2pt]{\scriptsize #2}\end{tabular}}
\begin{tabular*}{\linewidth}{@{\extracolsep{\fill}}lp{6.2cm}rrrr@{}}
\toprule
 & & \multicolumn{2}{c}{All trajectories} & \multicolumn{2}{c}{Test by outcome} \\
\cmidrule(lr){3-4}\cmidrule(l){5-6}
Spec.\ & $\mathrm{LTL}_f$ formula and reading & Model & Test & Failed & Successful \\
\midrule
$\tau^2$: $\chi_1$ & \rulecell{\neg\mathit{write}\;\mathsf W\;\mathit{detail}}{No write before a detail lookup} & 18.7 & 17.8 & 34.5 & 9.3 \\
$\chi_2$ & \rulecell{\mathsf G\neg(\mathit{write}\wedge\mathsf X(\neg\mathit{read}\;\mathsf U\;\mathit{write}))}{No two writes without a read in between} & 15.2 & 16.1 & 35.3 & 6.3 \\
$\chi_3$ & \rulecell{\mathsf G\neg(\mathit{write}\wedge\mathsf X\mathit{write}\wedge\mathsf X\mathsf X\mathit{write})}{Never three writes in a row} & 3.6 & 4.7 & 10.4 & 1.8 \\
\midrule
Terminal: $\psi_1$ & \rulecell{\mathsf G\neg(\mathit{warn}\wedge\mathsf X\mathit{warn})}{Never two format warnings in a row} & 12.8 & 13.9 & 18.0 & 5.3 \\
$\psi_2$ & \rulecell{\neg\mathsf F(\mathit{warn}\wedge\mathsf X\mathsf F(\mathit{warn}\wedge\mathsf X\mathsf F\,\mathit{warn}))}{At most two format warnings in total} & 10.4 & 11.4 & 15.1 & 3.9 \\
$\psi_3$ & \rulecell{\neg\mathit{complete}\;\mathsf W\;(\mathit{verify}\wedge\neg\mathit{complete})}{No completion claim before verification} & 8.3 & 8.5 & 11.6 & 2.1 \\
\bottomrule
\end{tabular*}
\par\smallskip
\parbox{\linewidth}{\scriptsize $\mathsf G$, $\mathsf F$, $\mathsf X$, $\mathsf U$, and $\mathsf W$ denote always, eventually, next, until, and weak until, respectively.}
\end{table}

%% file: tables/retail_spec_schemas.tex
\begin{table}[!ht]
\centering\small
\caption{\textbf{Tool-grounded retail specifications $\varphi_1$--$\varphi_5$.}
Safety formulas as compiled by \textsc{flloat}; propositions are evaluated per tool call against the latest preceding tool-confirmed facts for the bound order, item, or payment method.
$\mathit{known}$ holds when that evidence exists; the companion evidence property $\mathsf G(\mathit{call}\to\mathit{known})$ separates unknown from satisfied.}
\label{tab:retail_specs}
\setlength{\tabcolsep}{4pt}
\begin{tabular*}{\linewidth}{@{\extracolsep{\fill}}llp{5.2cm}@{}}
\toprule
& Safety formula & Reads as \\
\midrule
$\varphi_1$ & $\mathsf G\big((\mathit{attempt}\wedge\mathit{known})\to\mathit{status\_ok}\big)$ & the last confirmed status of the order permits the operation \\
$\varphi_2$ & \begin{tabular}[c]{@{}l@{}}$\mathsf G\big((\mathit{replace}\wedge\mathit{known})\to\mathit{equal}\big)$\\[-1pt]$\;\wedge\;\mathsf G\big((\mathit{modify}\wedge\mathit{known})\to\mathit{changed}\big)$\end{tabular} & replacement lists have equal length; pending modifications change the item \\
$\varphi_3$ & $\mathsf G\big((\mathit{item\_call}\wedge\mathit{known})\to\mathit{quantity\_ok}\big)$ & requested quantity does not exceed the order's confirmed count \\
$\varphi_4$ & $\mathsf G\big((\mathit{replace}\wedge\mathit{known})\to\mathit{available\_variant}\big)$ & replacement is an available variant of the same product \\
$\varphi_5$ & $\mathsf G\big((\mathit{return\_call}\wedge\mathit{known})\to(\mathit{original}\vee\mathit{gift})\big)$ & refund goes to the original payment or the owner's gift card \\
\bottomrule
\end{tabular*}
\end{table}

%% file: tables/temporal_requirement_case.tex
\begin{wraptable}{r}{0.5\textwidth}
\centering\small
\caption{\textbf{Temporal requirements for retail task 27.}
Failed: GPT-4.1; successful: o4-mini. V: first violation step; $+$: tool-confirmed acceptance.}
\label{tab:temporal-requirements}
\setlength{\tabcolsep}{2pt}
\begin{tabular*}{\linewidth}{@{\extracolsep{\fill}}lcc@{}}
\toprule
$\mathrm{LTL}_f$ specification & Failed & Successful \\
\midrule
$\varphi_{RE}=\mathsf G\neg(R_o^+\wedge\mathsf X\mathsf F E_o)$ & V: 12 & Holds \\
$\varphi_{ER}=\mathsf G\neg(E_o^+\wedge\mathsf X\mathsf F R_o)$ & Holds & V: 16 \\
$\mathsf F B_o^+$ & Unmet & Step 13 \\
\bottomrule
\end{tabular*}
\end{wraptable}

%% file: tables/task_requirement_case.tex
\begin{table}[!ht]
\centering\small
\caption{\textbf{Task requirements and their satisfaction in paired executions.}
The two test executions share retail task 103 and the agent configuration.
Entries give the step at which a successful target operation satisfies the property; ``Unmet'' means absent at termination.
$\mathsf F$ means eventually within the completed finite execution.}
\label{tab:task_requirements}
\setlength{\tabcolsep}{5pt}
\begin{tabular*}{\linewidth}{@{\extracolsep{\fill}}llcc@{}}
\toprule
$\mathrm{LTL}_f$ specification & Task requirement & Failed run & Successful run \\
\midrule
$\mathsf F r_b\wedge\mathsf F r_j$ & Request bookshelf and puzzle returns & Unmet & 12 \\
$\mathsf F r_k$ & Request backpack return & 8 & 13 \\
$\mathsf F a$ & Apply requested shipping address & 19 & 16 \\
$\mathsf F c$ & Change the specified item to red & 20 & 19 \\
\bottomrule
\end{tabular*}
\end{table}

%% file: figures/main/fig_retail_requirement_case.tex
\begin{figure}[!ht]
\centering
\resizebox{\linewidth}{!}{\input{figures/main/retail_requirement_case_tikz.tex}}
\caption{\textbf{Requirement checking on PrefixGuard replay paths.}
Selected steps from the same-task retail pair ($H{=}3$, seed 13).
Both paths reach $s_4$ at step 12, but the property monitor distinguishes a lookup from confirmed return requests.
The failed path ends with $\varphi_{\rm ret}=\mathsf F r_b\wedge\mathsf F r_j$ unmet; both runs satisfy the other operation requirements.
Dashed arrows omit intermediate DFA transitions; END is an added termination marker.}
\label{fig:retail-requirement-case}
\end{figure}

%% file: figures/main/retail_requirement_case_tikz.tex
\begin{tikzpicture}[
  x=1cm,y=1cm,font=\sffamily\fontsize{8.5}{10}\selectfont,
  >={Stealth[length=1.8mm,width=1.2mm]},
  state/.style={circle,draw=retailink,line width=.65pt,fill=white,
    minimum size=6.5mm,inner sep=0pt,font=\fontsize{10}{11}\selectfont},
  note/.style={align=center,inner sep=1pt},
  flow/.style={->,line width=.8pt},
  verdict/.style={rounded corners=2pt,draw,line width=.7pt,
    minimum width=2.75cm,minimum height=.85cm,align=center,inner sep=4pt}
]
\definecolor{retailink}{HTML}{46515B}
\definecolor{retailfail}{HTML}{B95824}
\definecolor{retailsuccess}{HTML}{30699C}

\fill[retailink!5,rounded corners=3pt] (4.05,.62) rectangle (6.05,-2.50);
\node[note,text=retailink,font=\sffamily\bfseries\fontsize{8.5}{10}\selectfont]
  at (5.05,.96) {Same DFA state};
\node[note,text=retailink,font=\sffamily\bfseries\fontsize{8.5}{10}\selectfont]
  at (12.95,.96) {Property at termination};

\node[note,text=retailfail,font=\sffamily\bfseries\fontsize{9}{10}\selectfont]
  at (0,0) {Failed\\run};
\node[state] (f0) at (1.40,0) {$s_0$};
\node[state] (f4) at (5.05,0) {$s_4$};
\node[state] (f8) at (9.60,0) {$s_8$};
\node[note,text=retailink] at (1.40,-.52) {$t=11$};
\node[note,text=retailfail] at (5.05,-.52) {Returns pending};
\node[note,text=retailink] at (9.60,-.52) {$t=21$};
\draw[flow,retailfail] (f0) -- node[note,above=4pt,text=retailink]
  {Look up target order\\$t=12,\ e_{14}$} (f4);
\draw[flow,retailfail,dashed] (f4) -- node[note,above=4pt,text=retailink]
  {Address and color changed\\steps 13--21} (f8);
\node[verdict,draw=retailfail,fill=retailfail!6,text=retailfail] (fend) at (12.95,0)
  {$\varphi_{\rm ret}$ violated\\Both returns still missing};
\draw[flow,retailfail] (f8) -- node[note,above=3pt,text=retailink] {END} (fend);

\node[note,text=retailsuccess,font=\sffamily\bfseries\fontsize{9}{10}\selectfont]
  at (0,-1.75) {Successful\\run};
\node[state] (s0) at (1.40,-1.75) {$s_0$};
\node[state] (s4) at (5.05,-1.75) {$s_4$};
\node[state] (s3) at (9.60,-1.75) {$s_3$};
\node[note,text=retailink] at (1.40,-2.27) {$t=11$};
\node[note,text=retailsuccess] at (5.05,-2.27) {Returns fulfilled};
\node[note,text=retailink] at (9.60,-2.27) {$t=20$};
\draw[flow,retailsuccess] (s0) -- node[note,above=4pt,text=retailink]
  {Submit both returns\\$t=12,\ e_{15}$} (s4);
\draw[flow,retailsuccess,dashed] (s4) -- node[note,above=4pt,text=retailink]
  {Other requirements fulfilled\\steps 13--20} (s3);
\node[verdict,draw=retailsuccess,fill=retailsuccess!6,text=retailsuccess] (send) at (12.95,-1.75)
  {$\varphi_{\rm ret}$ satisfied\\Both requests confirmed};
\draw[flow,retailsuccess] (s3) -- node[note,above=3pt,text=retailink] {END} (send);
\end{tikzpicture}